\documentclass{article}
\usepackage{graphicx} % Required for inserting images

 \PassOptionsToPackage{numbers, compress}{natbib}
 \usepackage{natbib}

\usepackage[utf8]{inputenc} % allow utf-8 input
\usepackage[T1]{fontenc}    % use 8-bit T1 fonts
\usepackage{hyperref}       % hyperlinks
\usepackage{url}            % simple URL typesetting
\usepackage{booktabs}       % professional-quality tables
\usepackage{amsfonts}       % blackboard math symbols
\usepackage{nicefrac}       % compact symbols for 1/2, etc.
\usepackage{microtype}      % microtypography
\usepackage{xcolor}         % colors
\usepackage{graphicx}
\usepackage{subcaption} 
\usepackage[utf8]{inputenc}
\usepackage{wrapfig} 

\usepackage{subcaption}
\usepackage{algorithm}
\usepackage{algpseudocode}

\usepackage[letterpaper,hmargin=1.0in,vmargin=1.0in]{geometry}

\usepackage{svg}

\usepackage[utf8]{inputenc} % allow utf-8 input
\usepackage[T1]{fontenc}    % use 8-bit T1 fonts

\usepackage{url}            % simple URL typesetting
\usepackage{booktabs}       % professional-quality tables
\usepackage{amsfonts}       % blackboard math symbols
\usepackage{nicefrac}       % compact symbols for 1/2, etc.
\usepackage{microtype}      % microtypography
\usepackage{xcolor}         % colors

\usepackage[american]{babel}
\usepackage[utf8]{inputenc} % allow utf-8 input
\usepackage[T1]{fontenc}    % use 8-bit T1 fonts
\usepackage{csquotes}
\usepackage{hyperref}       % hyperlinks
\usepackage{url}            % simple URL typesetting
\usepackage{booktabs}       % professional-quality tables
\usepackage{amsfonts}       % blackboard math symbols
\usepackage{nicefrac}       % compact symbols for 1/2, etc.
\usepackage{microtype}      % microtypography
\usepackage{xcolor}         % colors
\usepackage{array}          %nice tables
\usepackage{tikz}
\usetikzlibrary{calc,positioning,decorations.pathreplacing,shapes}
\usepackage{amsmath}
\usepackage{amsthm}
\usepackage{amssymb}
\PassOptionsToPackage{normalem}{ulem}
\usepackage{ulem}
\usepackage{dsfont}
\usepackage[noabbrev, capitalise]{cleveref}

\usepackage{enumitem}
\usepackage{xspace}
\usepackage{graphicx}
\usepackage{subcaption}

\theoremstyle{plain}
\newtheorem{theorem}{\protect\theoremname}[section]
  \theoremstyle{plain}
  
  \theoremstyle{plain}
  
  \theoremstyle{remark}
  
  \theoremstyle{plain}
  
  \theoremstyle{plain}
  \newtheorem{lemma}[theorem]{\protect\lemmaname}
  \theoremstyle{plain}
  
  \theoremstyle{plain}

\usepackage{mathtools}

\usepackage{contour}
\usepackage[normalem]{ulem}

\crefname{assumption}{Assumption}{Assumption}

\crefname{example}{Example}{Example}

\crefname{definition}{Definition}{Definition}

\makeatother
\crefformat{equation}{(#2#1#3)}

\numberwithin{equation}{section}

\contourlength{0.8pt}

\renewcommand{\underline}[1]{%
  \uline{\phantom{#1}}%
  \llap{\contour{white}{#1}}%
}

\definecolor{OliveGreen}{rgb}{0,0.6,0}
\definecolor{JaumeBlue}{rgb}{0,0,0.6}

\usepackage{babel}
  \providecommand{\conjecturename}{Conjecture}
  \providecommand{\corollaryname}{Corollary}
  \providecommand{\lemmaname}{Lemma}
  \providecommand{\questionname}{Question}
  \providecommand{\propositionname}{Proposition}
  \providecommand{\remarkname}{Remark}
\providecommand{\theoremname}{Theorem}

\newtheorem*{theorem*}{Theorem}

\usepackage[most]{tcolorbox}
\usepackage{algorithm}
\usepackage{nicefrac}       % compact symbols for 1/2, etc.
\usepackage{microtype}      % microtypography
\usepackage{xcolor}         % colors
\usepackage{graphicx}
\usepackage{subcaption} 
\usepackage[utf8]{inputenc}
\usepackage{wrapfig} 
\usepackage{subcaption}
\usepackage{algpseudocode}
\definecolor{darkgreen}{RGB}{46,125,50}
\definecolor{darkblue}{rgb}{0, 0, 0.5}
\hypersetup{colorlinks=true, citecolor=darkblue, linkcolor=darkblue, urlcolor=darkblue}
\definecolor{lightgreenbox}{RGB}{232,245,233}
\definecolor{darkgreenborder}{RGB}{46,125,50}

\newcommand{\ours}{Muon$^p$\xspace}

\usepackage{multirow}
\usepackage{siunitx}
\title{
\ours: Muon with Fractional Spectral Powers
}
\author{
Yihe Dong \\
Princeton University \\
\and
Will Sawin \\
Princeton University \\
}
\date{}

\begin{document}

\maketitle

\begin{abstract}
%We introduce PowerMuon, or \ours for short, a Muon-style optimizer parameterized by a spectral exponent $p$. 
%We show that for finetuning billion-parameter scale models, \ours achieves notably better downstream performance.
%We also analyze strength.
%Simple modification to Muon, same compute complexity.

Muon is an increasingly widely used optimizer that replaces a gradient $G=USV^\top$ with its polar factor $UV^\top$, thereby flattening the singular spectrum. However, full flattening discards singular-value information that may matter for adaptation. We introduce \ours, a Muon-style optimizer that instead uses fractional spectral-power updates $US^pV^\top$ for rational $p\in(0,1)$, interpolating between Muon and gradient descent. 
To make it practical, we prove that fractional spectral powers cannot be computed by any fixed univariate polynomial iteration, and furthermore derive low-degree odd bivariate recurrences that approximate $US^pV^\top$ using only matrix multiplications, preserving Muon’s matrix-multiplication-only structure and compute complexity. 
We show that \ours maximizes the linear improvement in loss under the Schatten $q$-norm for $q=1+\frac{1}{p}$.
Empirically, \ours is especially effective for finetuning: on billion-scale models, \ours improves validation perplexity and downstream task performance. We further analyze when \ours is less suitable, through the lens of spectral geometry. Our results reveal important insights on when preserving the singular spectrum can bring significant gains, and introduce a principled way to achieve them. Our code is publicly available at \url{https://github.com/princeton-pli/muon-p}.

%on Llama3.2-1B, $p=1/3$ improves validation perplexity over AdamW and Muon on Fineweb and Pile of Law, with gains that transfer to reasoning, calibration, and GSM8K. In contrast, Muon remains stronger in pretraining from scratch; exact-SVD ablations, a Muon$\rightarrow$PowerMuon curriculum, and effective-rank analyses indicate that the best spectral geometry is regime dependent. Uniform spectrum flattening helps broad feature discovery, whereas partially preserving singular-value structure improves downstream adaptation.
\end{abstract}

\section{Introduction}

%\ours generalizes Muon's SVD-free orthogonalization from approximating $UV^\top$ to approximating fractional spectral transforms $US^pV^\top$, for \textit{any} rational $p\in(0,1)$. We prove that this cannot be achieved by any fixed one-variable polynomial iteration, which forces a new two-variable polynomial recurrence formulation. We give a constructive family of odd low-degree recurrences whose matrix-multiplication-only iterates provably converge to the desired fractional spectral power.

%The main algorithmic challenge is that these fractional powers cannot be obtained by the fixed one-variable polynomial iteration used by Muon. We overcome this with a new odd bivariate polynomial recurrence that uses both the current iterate and the original matrix, yielding SVD-free, matrix-multiplication-only approximations with the same compute complexity as Muon.

Optimization algorithms play a central role in shaping how neural networks learn, influencing both the efficiency of training and the quality of the resulting model \citep{KingmaAdam, Loshchilov, jordan2024muon}.
%The way an optimizer reshapes the singular spectrum of a gradient matrix can determine whether learning is spread uniformly across directions or concentrated where the model most needs to adapt.
Muon has emerged as a promising optimizer for deep networks because its spectrum-flattening updates promote more balanced learning across singular directions, prevent dominant modes from monopolizing training, and have been shown to scale effectively in LLM pretraining \citep{jordan2024muon, liu2025scalablemuon, wang2025muontailend, shah2025practicalmuon, vasudeva2025muonspectral}. 

%Recent work has clarified that Muon's advantage is fundamentally spectral. Spectrum-aware updates can promote more balanced learning of the underlying components of the data \citep{muon_imbalanced_data}. 
%while a complementary spectral-family view places Muon at the $p=0$ endpoint of a broader class of updates of the form $US^pV^\top$ that continuously control spectral anisotropy. 
Muon achieves this by replacing a gradient $G = USV^\top$ with its polar factor $UV^\top$, which flattens the singular spectrum and makes the update more uniform across singular directions \citep{jordan2024muon}. 
But full flattening also discards singular-value information, suggesting a natural question: can we preserve Muon's matrix-level spectral benefits while retaining some magnitude structure that matters for adaptation?

We answer this question with $\mathrm{Muon}^p$, a Muon-style optimizer that updates with $US^pV^\top$ for rational $p \in (0,1)$. This family admits a clean geometric interpretation: $US^pV^\top$ is the \textit{steepest-descent direction under the Schatten-$q$ norm} with $q = 1 + 1/p$ (Theorem~\ref{steepest-descent}). %interpolating between standard gradient descent ($p=1$) and Muon ($p=0$). 
We identify and resolve key challenges to derive \ours and make its implementation practical. 
Unlike Muon's polar factor, fractional spectral powers cannot be obtained by iterating a fixed univariate polynomial; arbitrary rational powers require a recurrence that retains the original matrix (Lemma~\ref{lem:single_var_poly_insufficient}). We therefore derive low-degree odd bivariate recurrences that preserve Muon's matrix-multiplication-only structure (Theorem~\ref{odd-matrix-polynomial}), yielding practical updates that are a one-line change to the Muon implementation (Algorithm~\ref{algo:muon-p}).

While prior work \citep{delving_muon} has proposed fractional power for specific $p$ values, we introduce a systematic formulation from first principles, and explicitly derive \ours for \textit{any} fractional power $p$.

Empirically, \ours is most beneficial in finetuning. On a variety of models, \ours improves math and code reasoning, validation perplexity, and downstream task performance. 
We further analyze when \ours is less suitable than Muon.
%The picture is different in pretraining: from random initialization, Muon remains stronger, and exact-SVD ablations show that $UV^\top$ outperforms $US^{1/3}V^\top$. We interpret this regime dependence through an implicit direction-dependent learning-rate view: relative to Muon, $p>0$ reduces the effective step size along weaker singular directions, which is helpful when adaptation is concentrated in a smaller structured subspace, but too selective during broad feature discovery. Consistent with this view, a late-training Muon $\rightarrow$ PowerMuon curriculum lowers validation loss, and PowerMuon produces a sharper drop in effective rank. 
Taken together, these results position \ours as a principled generalization of Muon, while showing that the right spectral geometry is regime-dependent. %it extends spectrum-aware optimization beyond full orthogonalization with a constructive SVD-free method for rational fractional spectral powers, while showing that the right spectral geometry is regime-dependent---more uniform updates help pretraining, whereas partially preserved singular-value structure improves finetuning.

\paragraph{Contributions.}
Our key contributions are:
\begin{itemize}
    \item We introduce \ours, a generalized Muon-style optimizer that computes fractional spectral-power updates \(US^pV^\top\) for \textit{any} rational \(p \in (0,1)\) without explicit SVD.
    % \item We motivate \ours by showing it maximizes the linear improvement in loss under the Schatten $q$-norm for $q=1+\frac{1}{p}$.
    \item We show that no fixed one-variable polynomial iteration can compute these fractional powers, motivating a novel bivariate polynomial recurrence.
    \item \ours exhibits finetuning gains over Muon across math reasoning and coding, perplexity, and downstream task performance.
    \item We examine the strengths and weaknesses of \ours through the lenses of spectral learning-rate adaptation, curriculum, and effective-rank dynamics. 
\end{itemize}

\section{\ours: Muon with fractional spectral powers}

\paragraph{Motivation:}
 A neural network's optimizer computes, given the gradient $G$ of the loss function, an update to the weights. The simplest optimizer chooses a weight update linearly proportional to $G$. The AdamW optimizer updates each weight by a running average of the gradient in that weight and divides by a running average of the gradient squared. Muon takes advantage of the fact that both the gradient matrix and the weights can be interpreted as matrices. If $G = U S V^\top$ is the singular value decomposition of the gradient matrix $G$, then Muon updates the weights by adding $U V^\top$, so that weights update much more in the direction of small singular values, and less in the direction of large singular values, compared to the simplest optimizer.

For some applications, we argue that Muon is too extreme: We should not completely eliminate the information contained in the singular values, and should instead split the difference between Muon and the simplest optimizer. \ours achieves this by updating the weights by adding $U S^p V^\top$ for a chosen $p$ between $0$ and $1$. 

\paragraph{\ours maximizes loss improvement under Schatten $q$-norm.}
One theoretical justification for \ours comes from norm-based steepest descent: The Muon update rule maximizes the linear improvement in loss for a given size of the weight update, with the size measured by the operator norm~\citep{bernstein2025deriving}. \ours maximizes the linear improvement in loss for a given size of the weight update, with the size measured by Schatten $q$-norm for $q=1+\frac{1}{p}$. Theorem~\ref{steepest-descent} states this formally, with proof in Appendix~\ref{steepest-descent-proof}. The use of Schatten norms for steepest descent was also considered in \cite{Chen2025} and \cite{Cesista2025}.

\begin{theorem}\label{steepest-descent} For any real number $p \in (0,1)$ let $q=1+1/p$. Let $G$ be an $n\times n$ matrix with singular value decomposition $US V^\top$. Let $\langle, \rangle$ denote the entrywise dot product of matrices and let  $||\cdot||_q$ denote the Schatten $q$-norm. Then the maximum value of $\langle G,M\rangle$ among $n\times n$ matrices $M$ with $||M||_q$ bounded is attained by a scalar multiple of $ U S^{p} V^\top $. \end{theorem}

\paragraph{The computational problem:} \ours consists of two components: The choice of the update rule $U S^p V^\top$, and an algorithm to approximate $US^p V^\top$. To make \ours efficient, we require a fast algorithm to compute $U S^p V^\top$ given $U S V^\top$. The naive algorithm to compute $U S^p V^\top$ from $U S V^\top$, by performing a singular value decomposition, is too slow. The same difficulty occurs for computing $U V^\top$ in Muon. Muon gets around this by using an iterative algorithm to rapidly approximate $U V^\top$ given $G=USV^\top$. We show that a direct translation of their iterative method cannot be used to compute $U S^p V^\top$, but \ours uses a new variant method that can compute $U S^p V^\top$ for an arbitrary rational number $p \in (0,1)$. The remainder of this section is devoted to explaining this new iterative method.

\paragraph{Review of Muon formulation:} For $G$ a matrix and $G = US V^\top$ the singular value decomposition, where $S$ is a diagonal matrix with positive entries, the Muon optimizer uses a Newton-Schulz algorithm to approximate $U V^\top$ without computing the singular value decomposition. This algorithm is based on two observations:
\begin{enumerate}
    \item  We have $G G^\top G = U S^3 V^\top$ and $G G^\top G G^\top G G^\top G = U S^5 V^\top$. For any odd polynomial $f(x) = \sum_{i=0}^d a_{2i+1} x^{2i+1}$, these and similar identities allow one to compute $ U f(S) V^\top$ as $\sum_{i=0}^d a_{2i+1} G (G^\top G)^{i}$ using only matrix multiplication and not singular value decomposition.
 \item There are odd polynomials $f$ such that the iterations $f(x), f(f(x))$, $f(f(f(x))$ converge to $1$ for all $x$ in a starting range $(0,1]$.
\end{enumerate}

Thus, fixing any such odd polynomial $f$, after normalizing $G$ so that its greatest singular value is at most $1$, Muon iteratively applies the operation $ U S V^\top \mapsto U f(S) V^\top$ to approximate $U V^\top$, with the quality of the approximation improving as the number of iterations increases. The exact polynomial $f$ is chosen to make the time to compute a sufficiently good approximation as small as possible by balancing the number of iterations to attain convergence with the time to compute each iteration, which depends on the degree of the polynomial.

\paragraph{Our algorithm (introduction):} \ours applies a generalization of this strategy to compute $U S^{p} V^\top$ for an arbitrary rational number $p$ between $0$ and $1$. (A similar method could be applied for negative exponents or exponents greater than $1$.)  It is not possible to do this by iterating any operation of the form $ U S V^\top \mapsto U f(S) V^\top$:

\begin{lemma}[Iterating single-variable polynomials does not compute rational powers]\label{no-go} For any real number $p\in (0,1)$ not equal to $1$, there does not exist any polynomial in one variable $f$ such that, for all invertible matrices $G = U S V^\top$ with all singular values at most $1$, if we let $S_0 =S$ and $S_{n+1} = f(S_n)$ for all $n \geq 0$, then we have $\lim_{n\to\infty}  U f(S_n)  V^\top = U  S^{p} V^\top $. 
\label{lem:single_var_poly_insufficient}
\end{lemma}
For eigenvalues instead of singular values, the analogue of Lemma \ref{no-go} has been observed in the numerical linear algebra literature \cite{Guo2006,Higham2008}.

Instead, to converge to $U S^{p} V^\top$, an iterative algorithm must remember the initial matrix. Thus, we will consider a two-variable polynomial $f(x,y)$. If $G = U S V^\top$ is our initial matrix, \ours iteratively computes the matrices $G_n = U S_n V^\top$ where $S_0 = S$ and $S_{n+1}= f(S_n, S)$. Moreover, $S_n$ is chosen so that $G_n$ converge to $U S^{p} V^\top$. This can be done for $p$ an arbitrary \emph{rational} number in $(0,1)$. To show this is possible, we need two facts (whose proofs are deferred to the appendix):

\paragraph{First key fact (computing polynomials):} For $f(x,y)$ a polynomial in two variables which is odd in the sense that $f(-x,-y)=-f(x,y)$, given matrices $G_n= U S_n V^\top$ and $G= U S V^\top$, we can compute $ U f(S_n, S) V^\top$ using matrix multiplication and transpose. More precisely, we have the following theorem:
\begin{theorem}[Computing odd two-variable polynomials]\label{odd-matrix-polynomial}
\begin{enumerate} \item Let $f(x,y)$ be a polynomial in two variables with real coefficients which is odd. We can write $f$ in the form \begin{equation}\label{odd-polynomial-expression} f(x,y) = \sum_{i, j =0}^d a_{2i+1, 2j} x^{2i+1} y^{2j} + \sum_{i,j=0}^d a_{2i, 2j+1} x^{2i} y^{2j+1}\end{equation} for some natural number $d$ and tuples of real numbers $(a_{2i+1,2j})_{i,j=0}^d , (a_{2i,2j+1} )_{i,j=0}^d$.
\item For $f$ a polynomial given by the formula \eqref{odd-polynomial-expression}, for $G = U S V^\top$ and $G_n = U S_n V^\top$ we have 
    \[ U f(S_n, S) V^\top =  \sum_{i, j } a_{2i+1, 2j} G_n (G_n^\top G_n)^i ( G^\top G)^j  + \sum_{i,j} a_{2i, 2j+1}  G (G_n^\top G_n)^i (G^\top G)^j .\]
    \end{enumerate}
\end{theorem}

\paragraph{Second key fact (existence of polynomials):} There are odd polynomials $f(x,y)$ such that, for any $y\in (0,1]$, setting $x_0=y$ and $x_{n+1} = f(x_n,y)$ for all $n \geq 0$, we have $\lim_{n\to\infty} x_n = y^{p}$. More precisely, we have the following theorem:

\begin{theorem}[Odd two-variable polynomials compute rational powers]\label{main-existence} Let $p$ be a rational number in $(0,1)$. Then there exists an odd polynomial $f(x,y)$ such that, for any $y \in (0,1]$, if we set $x_0=y$ and $x_{n+1} = f(x_n,y)$ for all $n\geq 0 $, we have $\lim_{n\to\infty} x_n = y^{p}$. \end{theorem}

\paragraph{Our algorithm (conclusion):} For \ours, we choose a polynomial $f$ as in Theorem \ref{main-existence}. Given as input a matrix $G =  US V^\top$, we set $S_0 =S$, $ S_{n+1} = f(S_n, S) $ for all $n \geq 0$, and $G_n = U S_n V^\top  $ for all $n \geq 0$. We apply Theorem \ref{odd-matrix-polynomial} to compute $G_{n+1}$ from $G_n$ and $G$. By Theorem \ref{main-existence}, we have $\lim_{n\to\infty} S_n = S^{p}$ and thus $\lim_{n\to\infty} G_n = U S^{p} V^\top$. Hence, stopping after $N$ iterations for a fixed number $N$, we compute an approximation $G_N$ to $U S^{p} V^\top$. 

We obtain the original Muon algorithm as the special case of \ours where $p=0$ and $f(x,y)$ does not depend on $y$. As in Muon, \ours is numerically stable when implemented in bfloat16, allowing for fast implementation on modern hardware.

\begin{algorithm}[t]
\caption{\ours gradient update for $p=\frac{1}{3}$. Other values of $p$, including for Muon when $p=0$, can be easily accommodated by changing the \textcolor{darkgreen}{line 4} according to Equation~\eqref{eq:derivatives-special-form}.}
\begin{algorithmic}[1]
\State \textbf{Input}: gradient $G$, total number of steps $N$.
\State \textbf{Initialize} $G_0\gets G$
\For {$n \gets 0$ to $N-1$}
\State \textcolor{darkgreen}{$G_{n+1} = G_n + c ( G - G_n G_n^\top G_n)$}
\EndFor
\State \Return{\ours gradient update $G_N$.}  
\end{algorithmic}
\label{algo:muon-p}
\end{algorithm}

\paragraph{Key ideas in Theorem~\ref{odd-matrix-polynomial} and Theorem~\ref{main-existence} proofs.} The proof of Theorem \ref{odd-matrix-polynomial} is by applying the definitions and then canceling all unwanted factors of $U, U^\top, V, V^\top$ via the definition of orthogonal matrix. 

The proof of Theorem \ref{main-existence} requires constructing a suitable polynomial $f$. To do this, it is necessary to understand what criteria $f$ should satisfy. We should have $y^{p}$ be a fixed point of $f$ in the sense that $f(y^{p},y)=y^{p}$. Moreover, $y^{p}$ should be an attracting fixed point in that for $x$ close to $y^{p}$, $f(x,y)$ should be even closer to $y^{p}$. Finally, we must check convergence for our fixed initial value $x=y$.

Making $y^{p}$ a fixed point while taking $f$ odd requires taking a polynomial $f$ in a specific algebraic form. Making $y^{p}$ an attracting fixed point requires a certain inequality on the derivative of $f$, and so we find explicit $f$ in this algebraic form which satisfy the derivative inequality. We prove convergence from a given initial value using a monotonicity property, and we can choose $f$ to satisfy the monotonicity property also.

\section{Explicit polynomial construction}
\paragraph{Polynomial construction.}\label{sec-derivatives}

In this section, we explain how to construct an odd two-variable polynomial $f(x,y)$ such that, setting $x_{n+1} = f(x_n,y)$, $x_n$ converges to $y^{p}$ as $n$ goes to $\infty$ for suitably-chosen $x_0$. Since $p$ is a positive rational number, we write $p=a/b$ for positive integers $a,b$. We give the statements of the key results in the case that $a$ and $b$ are both odd, with the proofs, and the case when $a$ or $b$ is even, deferred to the appendix.

To construct this polynomial, it is necessary that $y^{a/b}$ be a fixed point, i.e. $f(y^{a/b}, y) = y^{a/b}$. This can be accomplished as long as $f$ takes the form 
\begin{tcolorbox}[
  colback=lightgreenbox,
  colframe=darkgreenborder,
  boxrule=0.8pt,
  arc=2pt,
  left=6pt,
  right=6pt,
  top=6pt,
  bottom=6pt
]
\begin{equation}\label{eq:derivatives-special-form} f(x,y) =x  + (y^a - x^b) g(x,y)\end{equation} 
\end{tcolorbox}
for some polynomial $g(x,y)$. (In practice, we will often take $g$ to be a constant, and to obtain convergence, it is necessary for this constant to be positive but not too large. We choose a constant in this range by empirical performance.) In fact, we have the following:

\begin{lemma}[Fixed point criterion]\label{divisibility-lemma} Let $a$ and $b$ be two coprime positive integers. Let $g(x,y)$ be a real polynomial in two variables. For $f(x,y)$ given by \eqref{eq:derivatives-special-form}, we have $f(y^{a/b}, y)= y^{a/b}$ for all $y\in (0,1]$.
\end{lemma}

Combining Lemma \ref{divisibility-lemma} with the following Lemma \ref{odd-even-lemma}, we can choose odd $f$ with $y^{a/b}$ a fixed point. 

\begin{lemma}[Oddness criterion]\label{odd-even-lemma} When $a$ and $b$ are odd, the polynomial $f(x,y)$ given by \eqref{eq:derivatives-special-form} is odd as long as $g(x,y)$ is even. 
\end{lemma}

However, for $x_n$ to converge (robustly) to $y^{a/b}$, it is not sufficient that $y^{a/b}$ be a fixed point of $f$. It must also be an attracting fixed point. The criterion for the existence of an attracting fixed point, based on the derivative, is as follows.

\begin{lemma}[Attracting fixed point criterion]\label{basin-of-attraction} Let $f$ be a polynomial in two variables and let $a$ and $b$ be two positive integers.   Say $y^{a/b}$ is an \emph{attracting fixed point} of $f(x,y)$ if $f(y^{a/b},y) = y^{a/b}$ and $$ \left|  \frac{\partial f}{\partial x} ( y^{a/b}, y) \right| <1. $$ 

If $y^{a/b}$ is an attracting fixed point of $f(x,y) $ then there exists an interval $I \subset \mathbb R$ containing $y^{a/b}$ such that, for any $x_0 \in I$, if we set $x_{n+1} = f(x_n,y)$ for all $n \geq 0$, then $\lim_{n\to\infty} x_n = y^{a/b}$.
\end{lemma}

Using this, we have an explicit criterion on the coefficients of $g(x, y)$ for when an odd polynomial has $y^{a/b}$ as an attracting fixed point.

\begin{lemma}[Choice of $g(x,y)$ coefficients]\label{derivative-criterion}Let $a$ and $b$ be two odd positive integers.  Let $f(x,y)$ be a polynomial given by \eqref{eq:derivatives-special-form} in terms of $g(x,y)$.  If 
    $$ 0< g( y^{a/b},y) < \frac{2}{ b y^{ \frac{ a(b-1)}{b}}} $$ for all $y\in (0,1]$ then $y^{a/b}$ is an attracting fixed point of $f(x,y)$ for all $y\in (0,1]$.
    
    In particular, for $g$ a constant $c$,  to have $y^{a/b}$ an attracting fixed point for all $y\in (0,1]$, it suffices to have $c< \frac{2}{b}$.  
    \label{lem:choice_of_c}
    \end{lemma}
%This is a general principle that can 
%\yd{WS Write short sentence on $d$}

As in Muon, we face a tradeoff where higher-degree polynomials may allow faster convergence but are slower to compute. For this reason, we choose $g$ to be the lowest degree possible, i.e. $g$ is a constant $c$.  As examples, some polynomials which arise from our formulas in the range of interest $0<a/b<1$ are as follows:

% $$ f(x,y) = x + c (y - x^3)  \textrm{ converging to } y^{1/3} $$
% $$ f(x,y) = x + c (y - x^5)  \textrm{ converging to } y^{1/5} $$
% $$ f(x,y) = x + c(y^3 - x^5)  \textrm{ converging to } y^{3/5} $$

\begin{tcolorbox}[
  colback=lightgreenbox,
  colframe=darkgreenborder,
  boxrule=0.8pt,
  arc=2pt,
  left=6pt,
  right=6pt,
  top=6pt,
  bottom=6pt
]
\begin{equation}
\begin{aligned}
f(x,y) &= x + c (y - x^3)   && \textrm{converging to } y^{1/3} \\
f(x,y) &= x + c (y - x^5)   && \textrm{converging to } y^{1/5} \\
f(x,y) &= x + (y^2 - x^4)(c y + d x) && \textrm{converging to } y^{1/2} \\
f(x,y) &= x + c (y^3 - x^5)   && \textrm{converging to } y^{3/5} \\
\end{aligned}
\label{eq:f_polynomial}
\end{equation}
\end{tcolorbox}
Empirically, we find that choosing, in the case $a=1, b=3$, $f(x,y) = x + c(y-x^3)$ for $c$ close to the largest value $2/3$ for which Lemma \ref{derivative-criterion} applies gives rapid convergence. Appendix~\ref{sec:c-values} provides further details and analysis on choosing coefficients for $f(x, y)$.
%\yd{generalizes Muon, when $p=0$, becomes scaled spectral norm}

Theorem \ref{main-existence} shows that we can in fact choose $f(x,y)$ so that we have convergence to $y^{a/b}$ from the starting value $x=y$. The proof of Theorem \ref{main-existence} uses a similar construction, but taking a smaller value of $c$. This gives a monotonicity property that makes it easier to give a formal proof of convergence. However, larger values of $c$ give convergence, and in fact more rapid convergence, in practice.

\section{\ours improves finetuning}
\label{sec:finetuning}

\subsection{Training}
\label{sec:training}
Unless stated otherwise, we train a Llama3-1B base model \citep{llama3}. 
\ours is a general framework that can approximate $G^p$ for arbitrary rational $p \in (0,1)$.  Unless noted otherwise, we adopt $p=\frac{1}{3}$ for ease of comparison, and the fact that Muon$^\frac{1}{3}$ can be efficiently computed by iterating a simple cubic polynomial. For language modeling, the data are packed into sequences of length 2048. Hyperparameters such as learning rate and batch size are tuned for each method. All training runs are performed on 2 to 4 H100 GPUs. Appendix~\ref{sec:training-details} contains further training and tuning details.

The polynomial $f(x,y)$ we use is given by the formula Eq~\ref{eq:derivatives-special-form} where the parameter $g(x,y)$ is either a constant or a linear polynomial, such as in Eq~\ref{eq:f_polynomial}. The optimal coefficients of $g(x,y)$ are chosen with a grid search over the range of possible values that yield convergence according to Lemma~\ref{lem:choice_of_c}. Appendix~\ref{sec:c-values} contains optimal coefficient values used for different exponents $p$.
We use the Muon implementation from NanoGPT Muon \citep{nanogpt}.
As in Muon, we accumulate a momentum of gradients $B_{t}=\mu B_{t-1} + G_t$, where $B_0=0$, $G$ is the gradient, and $\mu\in [0,1)$. We use the same $\mu$ in \ours as the $\mu$ from NanoGPT Muon \citep{nanogpt}.

\paragraph{Complexity.} 
\label{sec:complexity}
Let $N$ be the number of polynomial recurrence steps in \ours, $k$ the number of matrix multiplications (generally $\log_2(\text{deg}\ f )$ for $f$ in Eq~\ref{eq:derivatives-special-form}), and let $\omega$ denote the matrix multiplication complexity, the time complexity for both \ours and Muon is $O(\omega kN)$. Hence, \ours and Muon differ in complexity only by a constant factor, namely different numbers of matrix multiplications $k$. 
As in Muon, \ours is numerically stable in bf16. Section~\ref{sec:runtime-comparison} gives runtime comparisons in practice.

\subsection{\ours performance}

\paragraph{Math reasoning and coding.} Table~\ref{tab:gsm8k_merged} shows that the benefits of Muon$^p$ extend to mathematical reasoning and coding: when finetuning Llama3-1B base model on GSM8K \citep{cobbe2021gsm8k}, Numina \citep{numina_math_datasets}, and OpenCodeInstruct \citep{ahmad2025opencodeinstruct}, Muon$^p$ improves pass@8 validation accuracy (GSM8K and Numina) and loss (OpenCodeInstruct). %This suggests that preserving some spectral magnitude information is helpful for downstream tasks that require more structured, multi-step adaptation.
Figure~\ref{fig:numina-sweep} visualizes the performance comparison between \ours and Muon across learning rates on Numina.

%\yd{Muon could be overfitting in low data regime, analyze wrt batch size!.}

\begin{table}[h!]
\centering
\small
\setlength{\tabcolsep}{5pt}
\begin{tabular}{lccc@{\hspace{1.0em}\vrule width 0.4pt\hspace{1.0em}}ccc}
\toprule
& \multicolumn{3}{c}{\textbf{Llama3-1B}}
& \multicolumn{3}{c}{\textbf{Pretrained w/ Muon}} \\
\cmidrule(r{0.9em}){2-4}
\cmidrule(l{0.7em}){5-7}
& AdamW & Muon & \ours & AdamW &Muon & \ours \\
\midrule
GSM8K $\uparrow$ 
& 46.41 &41.76 & \textbf{48.27} 
& 11.35 & 8.114  & \textbf{12.88} \\
OpenCodeInstruct $\downarrow$ 
& 0.2398 &0.2423 & \textbf{0.2382} 
& 0.653 & 0.746 &\textbf{0.633} \\

Numina $\uparrow$ 
& 40.75& 39.67& \textbf{42.41} 
& 9.14 & 5.67 &\textbf{10.45} \\
\bottomrule
\end{tabular}
\caption{Comparison of \ours{}, Muon, and AdamW after finetuning either Llama3-1B or a Muon-pretrained model on Fineweb. 
%$\mathrm{pass}@8$ validation accuracy when finetuning Llama3-1B base model on GSM8K for 10 epochs.
}
\label{tab:gsm8k_merged}
\end{table}

%\paragraph{Sweep}
\begin{wrapfigure}{h}{0.4\textwidth}
    \centering
    \vspace{-0.5em}
    \includegraphics[width=\linewidth]{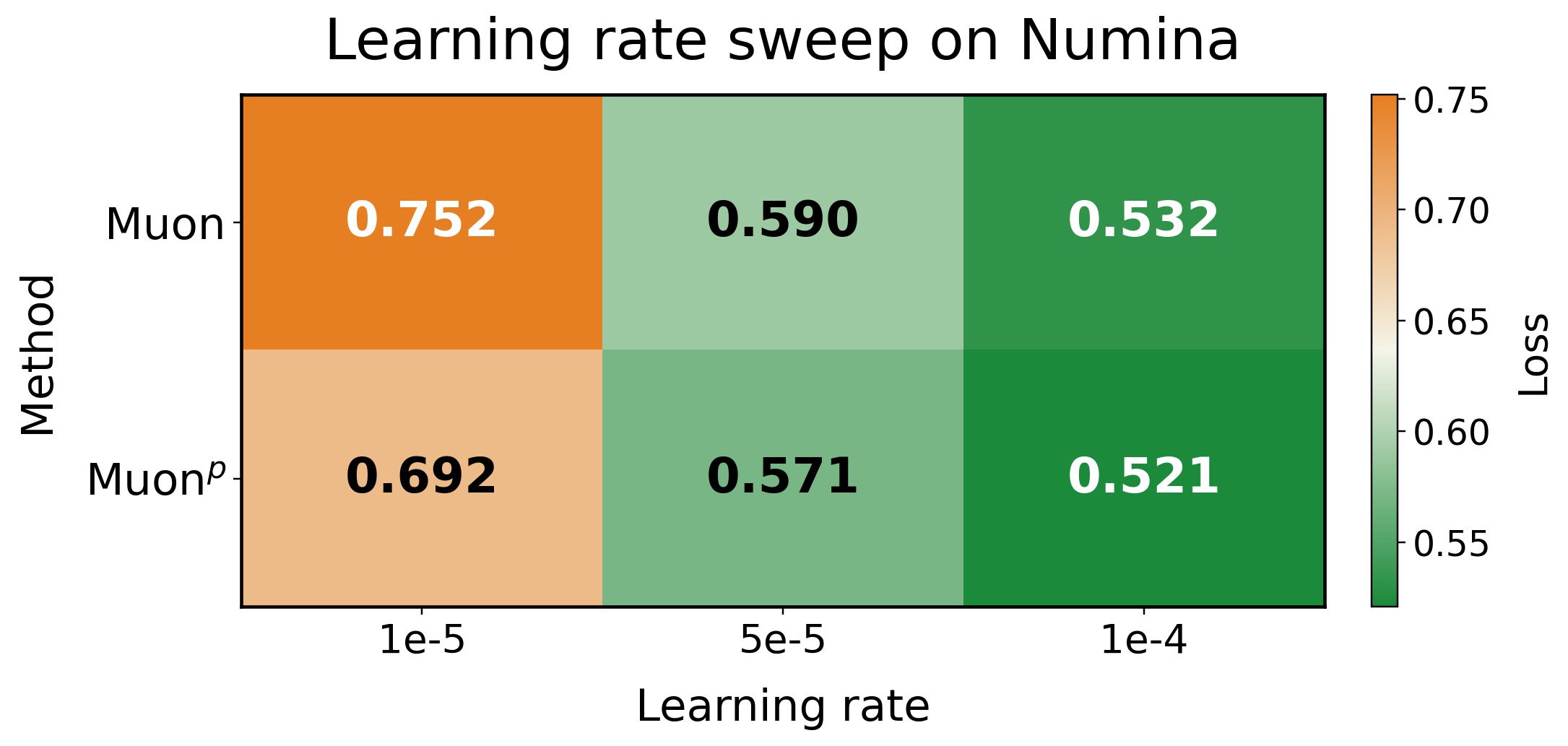}
    \vspace{-0.8em}
    \caption{\ours outperforms Muon across learning rates.}
    \vspace{-0.8em}
    \label{fig:numina-sweep}
\end{wrapfigure}
Appendix~\ref{sec:additional-baselines} compares \ours with additional baselines.

\paragraph{Using Muon-pretrained checkpoints.}
Since Llama3-1B was pretrained using AdamW, finetuning on Muon or \ours creates an optimizer mismatch, which can harm finetuning performance \citep{optimizerModelConsistency, canMuonFinetuneAdam}.
To study this effect, we pretrained a randomly initialized Llama3-1B model using Muon, on 9.83 billion tokens of Fineweb. We then finetuned it using AdamW, Muon, and \ours. 

Table~\ref{tab:gsm8k_merged} shows that despite matching the pretrain optimizer (and having an advantage), Muon still underperforms \ours and Adam, suggesting that \ours's superior finetuning performance on the off-the-shelf pretrained Llama3-1B model results from more fundamental reasons than optimizer mismatch.
Indeed, an analysis using exact SVD, in Section~\ref{sec:exact-svd}, shows that pretraining inherently benefits from full-spectrum updates, in contrast to finetuning.

%\paragraph{General domain data.}

%\yd{gains even on domains similar to pretrain}

\begin{table}[h]
\centering
\small
\begin{tabular}{lccc}
\toprule
 & AdamW & Muon & \ours{} \\
\midrule
\textbf{Law} $\downarrow$ & 2.591 & 2.578 & \textbf{2.491} \\
\textbf{Fineweb} $\downarrow$ &\textbf{11.57} & 11.70 & \textbf{11.58} \\
\bottomrule
\end{tabular}
\caption{Validation perplexity for Llama3-1B base model finetuned on Fineweb and Pile of Law.
}
%\vspace{-20pt}
\label{tab:main_perplexity}
\end{table}

Section~\ref{sec:larger-scale} contains results on a larger scale with Llama3-3B.

\paragraph{Commonsense reasoning and other language tasks.} 
Table~\ref{tab:main_perplexity} shows that \ours yields perplexity gains when finetuning Llama3-1B on both 300 million tokens of Fineweb \citep{lozhkov2024fineweb-edu} and 625 million tokens of Pile of Law \citep{pile_of_law}. The best hyperparameters are tuned for each method and task. All runs use the same seed and data ordering, with optimal hyperparameters tuned for each method.

%The downstream task performance echo the perplexity results.
Table~\ref{tab:downstream_evals} shows that the gains from \ours extend beyond perplexity to accuracy in a range of downstream tasks. 

\begin{table*}[htb!]
\centering
\small
\begin{tabular}{lcccccccc@{\hspace{6pt}}|@{\hspace{6pt}}c}
\toprule
& PIQA & WG & LAMBADA & BoolQ & OBQA & HellaSwag & ARC-E & ARC-C & Avg. \\
%& acc\_n & acc & acc & acc & acc\_n & acc\_n & acc\_n & acc\_n & avg.\\
\midrule
AdamW 
& 75.24 
& 61.25 
& \textbf{61.85} 
& 64.43 
& 28.00 
& \textbf{48.23} 
& \textbf{67.38} 
& 32.34 
& 54.84 \\
Muon 
& 75.19 
& 61.01 
& 61.46 
& 65.26 
& 27.4 
& 48.21 
& 67.17
& 32.17 
& 54.73 \\

\ours 
& \textbf{75.41} 
& \textbf{61.33} 
& 61.81 
& \textbf{65.32} 
& \textbf{28.20} 
& 48.18 
& 67.21 
& \textbf{32.76} 
& \textbf{55.03} \\
\bottomrule
\end{tabular}
\caption{\ours improves upon Muon in accuracy across a variety of reasoning and language tasks. 
WG: Winogrande. OBQA: OpenBookQA. ARC-E: ARC-Easy; ARC-C: ARC-Challenge.
}
\label{tab:downstream_evals}
\vspace{-1mm}
\end{table*}

\paragraph{Runtime comparison.}
\label{sec:runtime-comparison}
We randomly initialized a Llama3-1B model in bf16, and benchmarked the runtime of a single `optimizer.step' call, using a single H100 GPU and identical data. Each optimizer is allowed to warm up for 3 steps (for `torch.compile', etc), and measured the mean runtime of the subsequent 15 steps.

As shown in Table~\ref{tab:optimizer-step-time},
\ours with $p=\frac{1}{5}$ has a comparable runtime with Muon, which is not surprising as both use a quintic polynomial recurrence. \ours with $p=\frac{1}{3}$ uses a cubic recurrence, leading to faster runtime. 

\begin{table}[h]
\centering
\begin{tabular}{l S[table-format=3.2]S[table-format=1.2]}
\toprule
\textbf{Backend} & \multicolumn{2}{c}{\textbf{Optimizer step time (ms)}} \\
\cmidrule(lr){2-3}
 & {\textbf{Mean}} & {\textbf{Std.}} \\
\midrule
$\mathrm{Muon}^{p}$, $p=\frac{1}{3}$ & 101.89 & 1.22 \\
$\mathrm{Muon}^{p}$, $p=\frac{1}{5}$ & 116.30 & 0.43 \\
$\mathrm{Muon}$                       & 120.73 & 0.26 \\
\bottomrule
\end{tabular}
\caption{Optimizer step time for \ours and Muon.}
\label{tab:optimizer-step-time}
\end{table}

\section{\ours analysis}

\paragraph{Pretraining benefits from uniform full spectrum updates.}

Interestingly, \ours performs worse than Muon when trained directly from a randomly initialized model, i.e. during pretraining. This contrasts with the results in Section~\ref{sec:finetuning}, as shown in Figure~\ref{fig:pretrain_vs_finetune} for pretraining and finetuning the 135M SmolLM base model \citep{bakouch2025smollm3}. We analyze this dichotomy below.

\begin{figure}[!h]
    \centering
    \includegraphics[width=.7\linewidth]{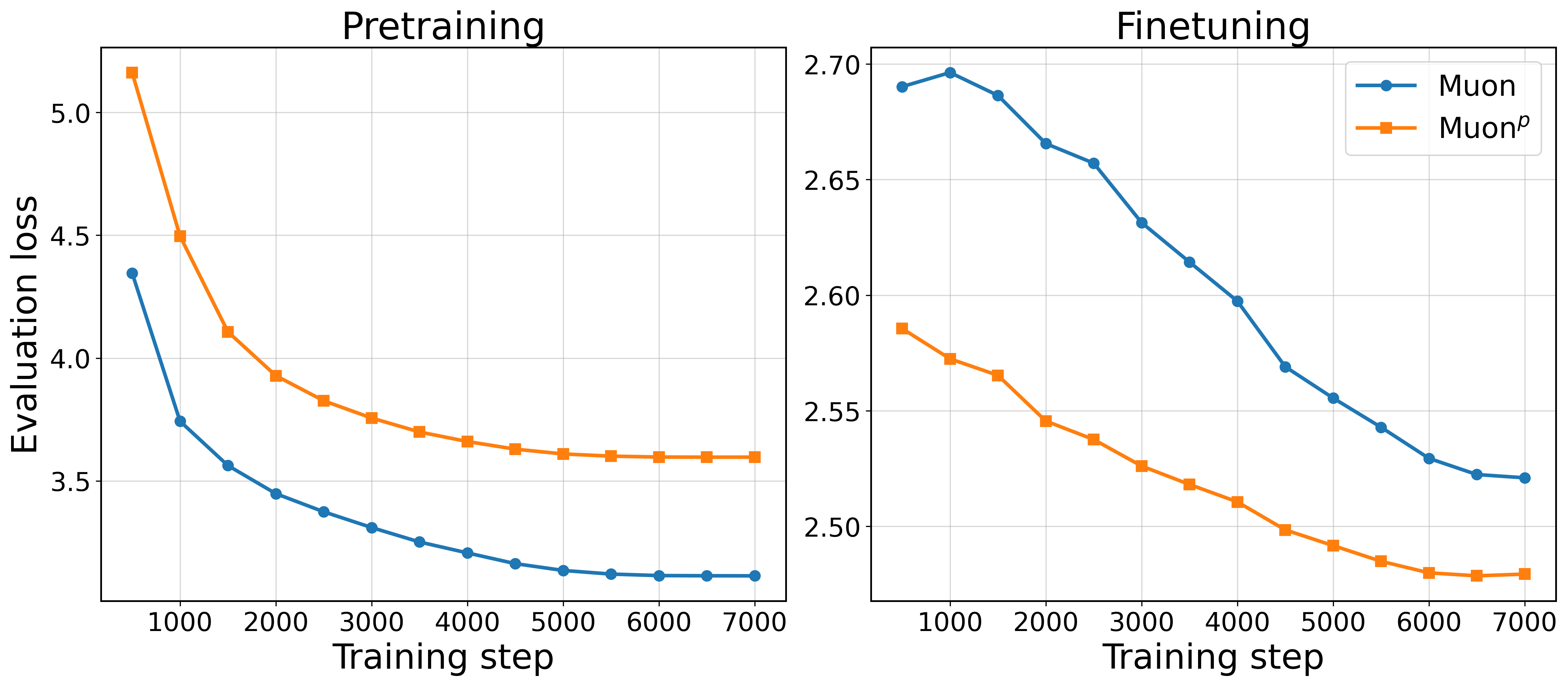}
    \caption{
    Pretraining vs finetuning for Muon and \ours.
    }
    \label{fig:pretrain_vs_finetune}
    \vspace{-1mm}
\end{figure}

\paragraph{Pretraining.}
During pretraining, where the objective is to discover and refine features broadly, it is beneficial to learn the data's principal components more uniformly, as Muon does, rather than allowing a small number of dominant directions to absorb most of the update. Indeed, recent works suggest that this isotropization especially helps associative-memory layers and tail patterns, because it prevents high-frequency or high-gain directions from monopolizing the update \citep{muon_imbalanced_data, muon_associative_recall, muon_dynamics_training}.

\paragraph{Finetuning.}
In contrast, after pretraining, adaptation commonly lives in a much smaller and more structured subspace \citep{intrinsic_finetune_dim}. Pretrained models have low intrinsic dimension for downstream tuning, full finetuning tends to preserve much of this pretrained spectral structure \citep{pome, si2025weightspectra, hwang2025pica, shuttleworth2024lorafft, SVFT}; and parameter-efficient methods that explicitly update in a low-dimensional subspace work surprisingly well \citep{hu2021lora, meng2024pissa, wang2024loraga, jiang2024mora, lialin2023relora, SVFT}. In other words, the singular values encode which residual directions are actually worth changing. In this regime, Muon's gradient updates with $p=0$ discards all this directional information, while using $0< p < 1$ keeps Muon’s conditioning benefit but restores some preference for the stronger singular directions.

\paragraph{A direct empirical test with exact SVD.}
\label{sec:exact-svd}
A direct way to test this hypothesis is to replace the approximate spectral powers used during pretraining with the exact ones, computed from the gradient's SVD. As shown in Figure~\ref{fig:exact_svd_pretrain_finetune}, the update \(UV^\top\) outperforms \(US^{1/3}V^\top\) in pretraining. This suggests that \ours underperforms Muon in pretraining not because it poorly approximates spectral-power updates, but because pretraining itself benefits from uniform updates across the spectrum.

However, retaining partial singular value information is useful when finetuning on specific domains, as evident in Figure~\ref{fig:exact_svd_pretrain_finetune} (right figure).
Remarkably, this adaptation is on a checkpoint \textit{pretrained using Muon}, meaning that despite having an advantage in terms of optimizer consistency \citep{canMuonFinetuneAdam, optimizerModelConsistency}, discarding spectrum information with $US^0 V^\top$ updates still underperforms.

\begin{figure}[h]
    \centering    \includegraphics[width=0.9\linewidth]{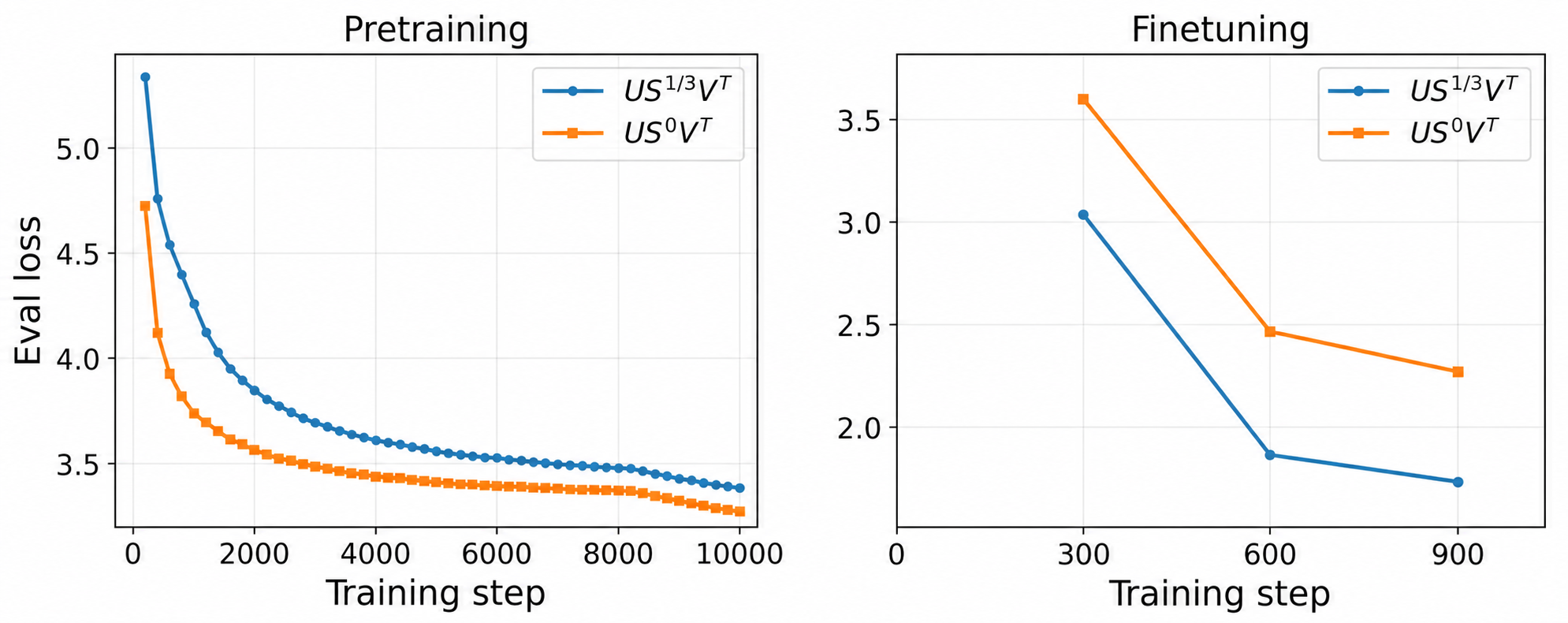}
    \caption{\textbf{Left}: Updating with exact SVD shows that pretraining benefits from uniform updates across the  spectrum. \textbf{Right}: However, retaining singular value information helps adaptation on specific domains, in this case OpenCodeInstruct.} \label{fig:exact_svd_pretrain_finetune}
    \vspace{-1mm}
\end{figure}

\paragraph{Effects of different exponent $p$.}  Figure~\ref{fig:diff_exp_p_sweep} shows the validation loss after finetuning SmolLM-360M on 852 million tokens of Fineweb data with \ours for different values of exponent $p$, indicating that the optimal exponent is neither $p=0$ nor $p=1$, but rather somewhere in-between. These results also speak to the robustness of \ours, as the validation performances for different $p$ differ by at most 0.01.

\paragraph{Implicit direction-dependent learning rate adaptation.}
The spectral-power update in \ours can be viewed as introducing an implicit learning rate scaling factor that varies across singular directions of the gradient. Reweighting the singular values by a power $p > 0$ attenuates smaller singular directions more strongly, while leaving dominant directions relatively less affected. In particular, the effective learning rate follows: \( \Delta = -\eta\,US^p V^\top = -\sum_i \eta_i^{\mathrm{eff}}\,\sigma_i\,u_i v_i^\top \), where \( \eta_i^{\mathrm{eff}} = \eta\, \sigma_i^{p-1} \).

\begin{figure}[h]
    \centering    \includegraphics[width=0.5\linewidth]{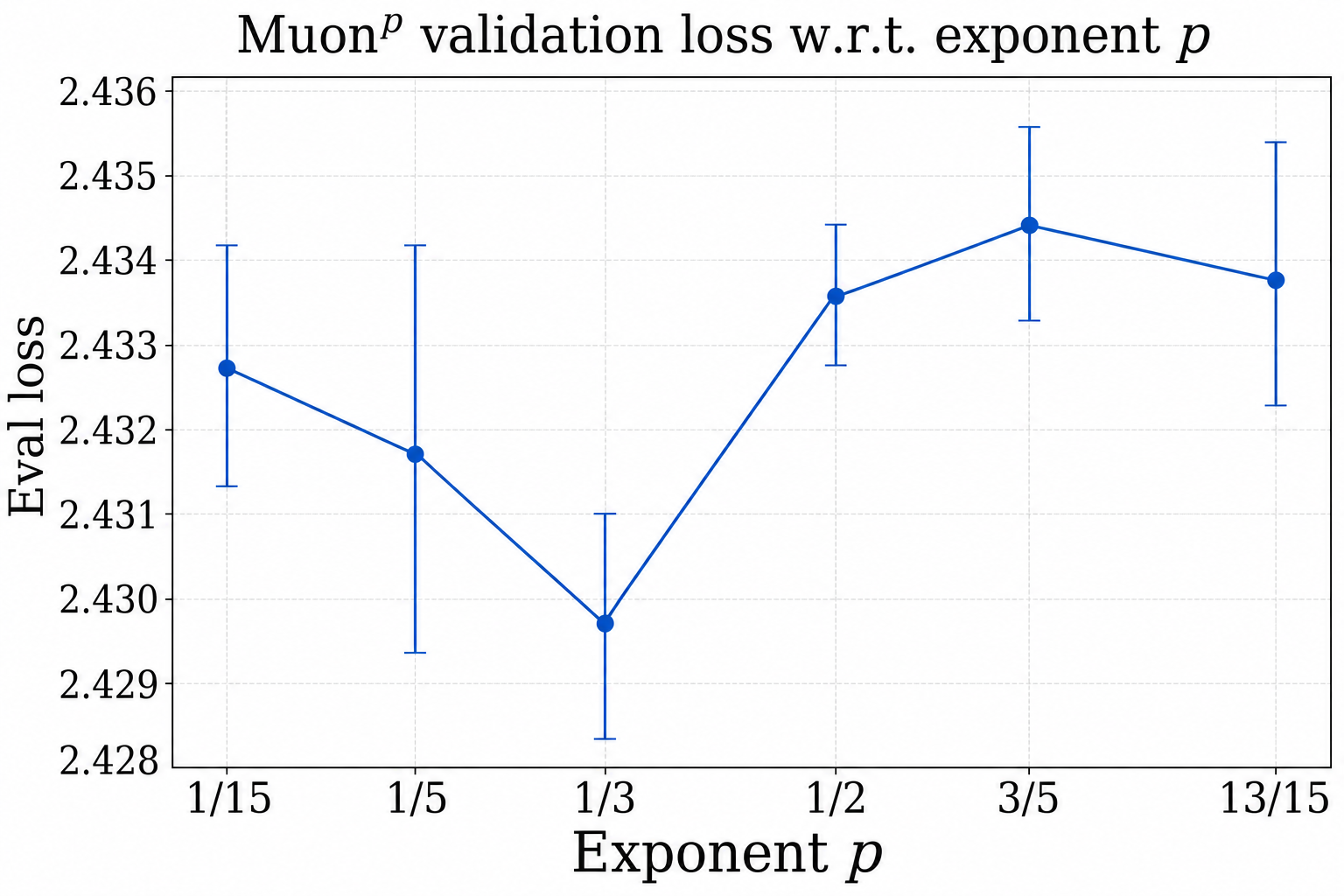}
    \caption{The optimal validation loss is neither $p=0$ nor $p=1$, but rather somewhere in-between. For comparison, the Muon validation loss after training on the same data is $2.52$.}
    \label{fig:diff_exp_p_sweep}
    \vspace{-1mm}
\end{figure}

%\[
%\Delta = -\eta\, U \operatorname{diag}\!\bigl(\sigma_i^p\bigr) V^\top
%= -\sum_i \underbrace{\eta\,\frac{\sigma_i^p}{\sigma_i}}_{\eta_i^{\mathrm{eff}}}\,\sigma_i\,u_i v_i^\top.
%\]
%\[
%\begin{aligned}
%\Delta &= -\eta\, U \operatorname{diag}\!\bigl(\sigma_i^p\bigr) V^\top, \\
% &= -\sum_i \underbrace{\eta\,\frac{\sigma_i^p}{\sigma_i}}_{\eta_i^{\mathrm{eff}}}\,\sigma_i\,u_i v_i^\top.
%\end{aligned}
%\vspace{-1mm}
%\]
As the gradient in standard Muon implementations \citep{nanogpt} is spectral normalized such that $\sigma_i \le 1$ for all $i$, $p > 0$ means that the effective learning rate $\eta^{\mathrm{eff}}$ for \ours in the lesser singular directions is less than that for Muon, which corresponds to $p=0$.
%\yd{justify using small lr later in training}

\paragraph{Muon $\to$ \ours curriculum}
\begin{wrapfigure}{!h}{0.5\textwidth}
    \centering
    \vspace{-0.5em}
    \includegraphics[width=\linewidth]{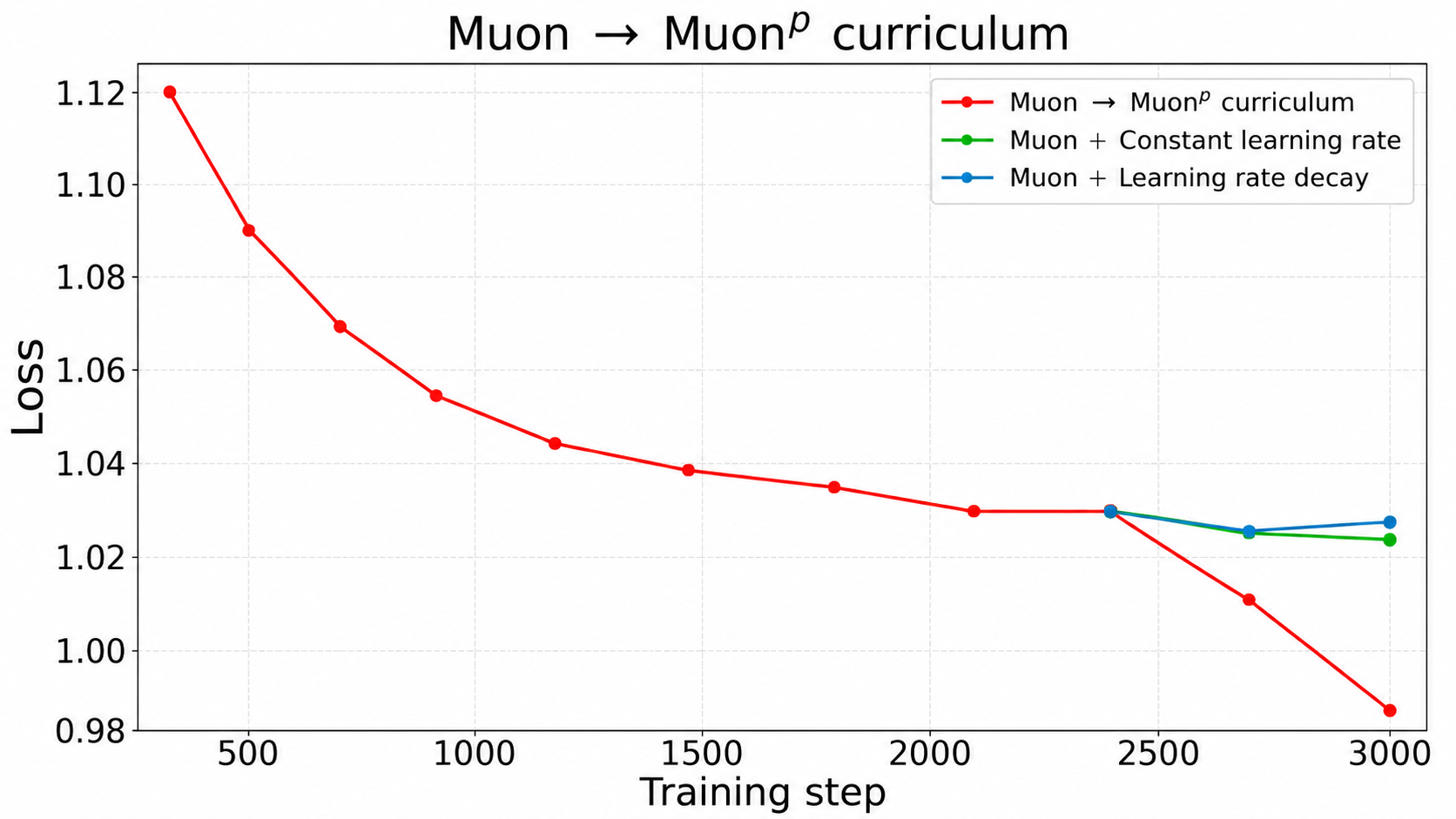}
    \vspace{-0.8em}
    \caption{
    Eval loss vs.\ step for Muon and Muon $\rightarrow$ \ours curriculum, compared with reducing the learning rate or keeping it constant. 
    The loss drop in the final 500 steps highlights the effects of the curriculum. All three runs use the exact same seed and data ordering.
    }
    \vspace{-0.8em}
    \label{fig:curriculum_eval_loss}
\end{wrapfigure}
To empirically test this adaptive learning-rate interpretation, we impose an optimizer curriculum during pretraining: Muon is used for most of training, and then switched to \ours for the final 500 steps. This isolates the effect of increasing the spectral power from \(0\) to \(p\), which corresponds to reducing the effective learning rate \(\eta^{\mathrm{eff}}\) in the smaller singular directions of the gradient.

As shown in Figure~\ref{fig:curriculum_eval_loss}, the validation loss readily decreases once the spectral power increases from $p=0$ to $p=\frac{1}{3}$. This happens consistently across model scales.
Our work paves way to exciting future directions to translate such curriculum settings into significant practical gains for pretraining.

In practice, the transition from Muon to \ours is seamless for any choice of \(p\), since Muon is a special case of \ours. Only the polynomial approximation needs to be changed, while the rest of the optimizer state, such as gradient momentum, carries over directly. This is a notable advantage over switching to a different optimizer family, such as AdamW, which requires maintaining different optimizer state, including both first and second moment estimates.

\paragraph{Effective rank analysis}

The interpretation above suggests that spectral-power updates should not only change optimization performance, but also change the geometry of the learned updates. If increasing $p$ suppresses smaller singular directions and concentrates more weight on dominant ones, then the optimizer should effectively operate in a lower-dimensional subspace over training. Effective rank provides a natural way to test this hypothesis.

The effective rank of a matrix \(W\) is defined as \(\mathrm{eff\_rank}(W)=\exp\!\bigl(-\sum_{i=1}^r p_i \log p_i\bigr)\), where \(p_i=\sigma_i/\sum_{j=1}^r \sigma_j\) and \(\sigma_1,\dots,\sigma_r\) denote the nonzero singular values of \(W\).
As shown in Figure~\ref{fig:erank}, Muon maintains a relatively higher effective rank and decays slowly, while \ours shows a much sharper drop, especially early on, and continues decreasing throughout training.

\begin{figure}[h]
    \centering
    \includegraphics[width=0.45\linewidth]{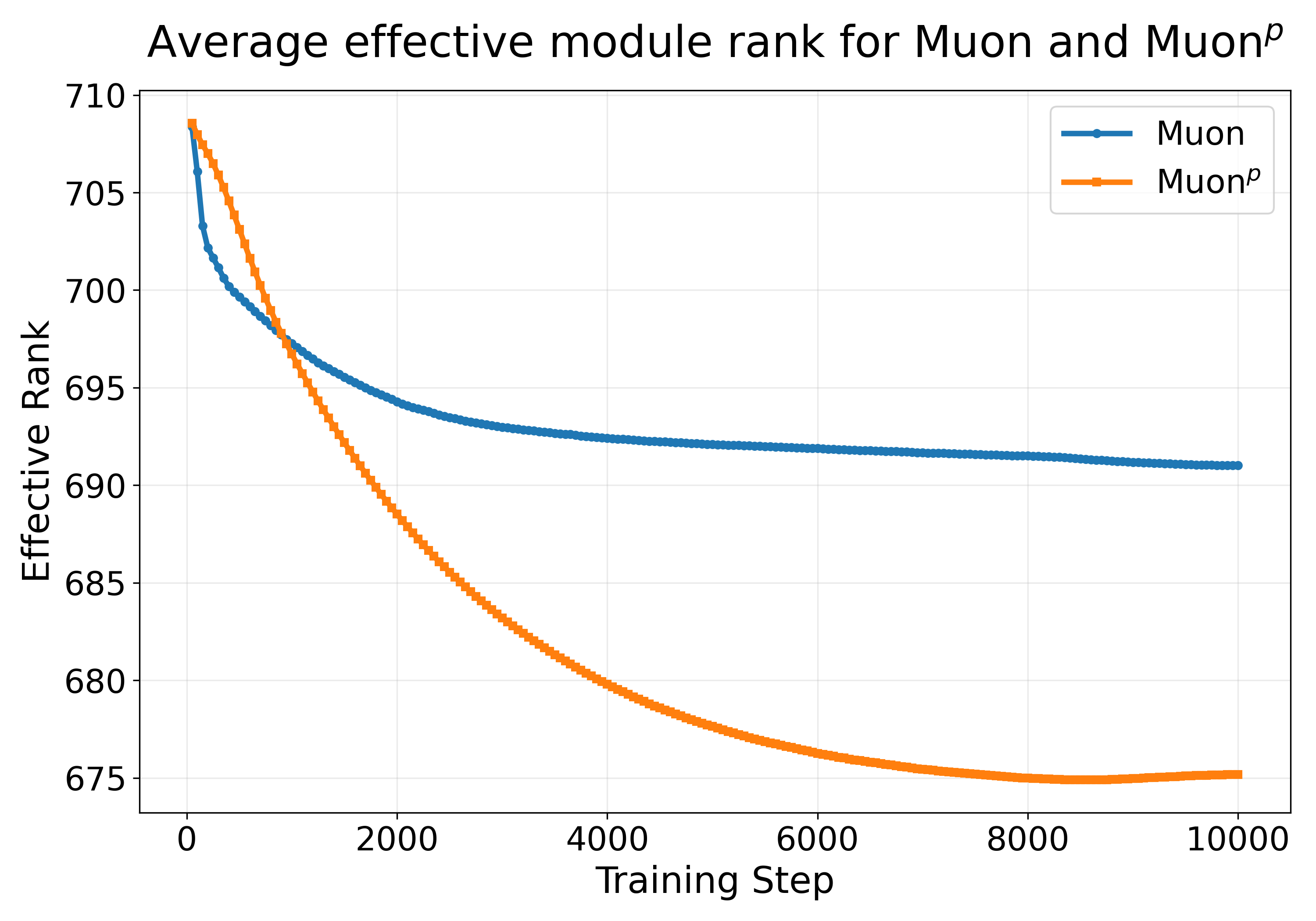}
    \caption{
    Average effective module rank for Muon and \ours during training (from scratch). 
    \ours exhibits a significantly stronger reduction in effective rank over training compared to Muon.
    }
    \vspace{-1mm}
    \label{fig:erank}
\end{figure}

\section{Related work}

%\paragraph{Matrix-aware and norm-based optimizers.}
%The dominant baselines for Transformer training remain coordinate-wise adaptive methods such as Adam and AdamW \citep{kingma2014adam, loshchilov2017adamw}, while structured preconditioners such as K-FAC, Shampoo, and SOAP exploit matrix or tensor geometry more directly \citep{martens2015kfac, gupta2018shampoo, vyas2024soap}. Muon pushes this idea further by replacing a matrix-valued momentum update with an approximation to its polar factor $UV^\top$, thereby discarding singular-value magnitudes instead of merely rescaling them \citep{jordan2024muon, bernstein2025deriving}. Related norm-based perspectives interpret Muon as part of a broader non-Euclidean design space, including the modular-norm program and LMO-based optimizers such as Scion and Gluon \citep{bernstein2024anthology, large2024modularnorm, bernstein2024modularduality, pethick2025scion, riabinin2025gluon}. \ours fits naturally into this view: rather than changing where Muon is applied, it introduces a spectral exponent $p$ that continuously interpolates between the fully equalized Muon update and updates that preserve more of the original singular-value anisotropy.

\paragraph{Muon theory and extensions.}
Muon has rapidly accumulated both empirical and theoretical follow-up. On the empirical side, Muon has been shown to scale to LLM pretraining with appropriate update rescaling and regularization \citep{liu2025muonscalable, shah2025practicalmuon}. On the theory side, existing work analyzes its convergence or interprets it as steepest descent under operator/spectral norm constraints \citep{li2025muonnote, kovalev2025orthogonalization, shen2025muonconvergence, chen2025spectralconstraints}. Task-specific analyses further explain when spectrum flattening helps: Muon improves tail-end associative-memory learning and exhibits favorable scaling laws in that setting \citep{muon_associative_recall, muon_dynamics_training}, while imbalanced-data studies argue that equalizing singular directions prevents dominant modes from monopolizing learning \citep{muon_imbalanced_data}. These results resonate with our finding that $p=0$ remains stronger in pretraining, where broad feature discovery benefits from more uniform full-spectrum updates, whereas $0<p<1$ is better suited to finetuning.

The closest prior work to ours is \citet{delving_muon}, which studies a spectral family of Muon-like transforms and develops SVD-free coupled Newton--Schulz methods for selected intermediate powers. Two main differences 
are that \ours provides an explicit construction for arbitrary matrix power $p\in (0,1)$, and \ours does \textit{not} use coupled Newton Schulz iterations for matrix inverse roots, which requires float32 precision to be numerically stable \citep{jordan2024muon}.
Other recent variants modify Muon's spectral transform or robustness properties through adaptivity, heavy-tailed corrections, or alternative whitening surrogates, including AdaMuon, NorMuon, HTMuon, ROOT, PolarGrad, and MUD \citep{delving_muon, si2025adamuon, li2025normuon, pang2026htmuon, he2025root, lau2025polargrad, southworth2026mud}. \ours is complementary to these works: instead of proposing a new heuristic transform or a different whitening surrogate, it isolates the fractional spectral power itself as the control variable, proves a negative result for fixed one-variable polynomial iterations, and gives constructive odd bivariate recurrences for arbitrary rational $p\in(0,1)$.

\paragraph{Efficient computation of orthogonal and spectral transforms.}
A separate line of work revisits the numerical linear algebra subroutines underlying Muon. Muon's practical implementation relies on Newton--Schulz-type iterations that use only matrix multiplications \citep{jordan2024muon}. Recent papers improve this subroutine through optimized polynomial coefficients (CANS), adaptive minimax iterations (Polar Express), iteration-free consolidations (IFNSO), and general adaptive acceleration frameworks (PRISM) \citep{grishina2025cans, amsel2025polarexpress, hu2026ifnso, yang2026prism}. Classical high-accuracy polar-decomposition methods such as QDWH and Zolo-PD offer a useful point of comparison, but they typically depend on QR or rational-function machinery rather than the matrix-multiplication-only structure favored in deep learning workloads \citep{nakatsukasa2010qdwh, nakatsukasa2016zolotarev}. \ours addresses an orthogonal computational problem: not how to compute the polar factor $UV^\top$ faster, but how to compute fractional transforms $US^pV^\top$ without explicit SVD while retaining Muon's accelerator-friendly structure. %Our two-variable polynomial construction is therefore complementary to faster polar-decomposition methods rather than a replacement for them.

%\paragraph{Spectral structure in finetuning and low-dimensional adaptation.} Our empirical motivation also connects to work arguing that downstream adaptation is both low-dimensional and spectrally structured. Intrinsic-dimension results show that successful finetuning often lives in a much smaller subspace than the ambient parameter space \citep{intrinsic_finetune_dim}. More recent analyses show that full finetuning largely preserves the pretrained spectrum while selectively amplifying a small number of dominant singular directions, and that methods aligned with the pretrained singular subspace more closely mimic full finetuning \citep{si2025weightspectra, hwang2025pica, shuttleworth2024lorafft, SVFT}. LoRA and related PEFT methods exploit this structure explicitly through low-rank or singular-vector parameterizations \citep{hu2021lora, meng2024pissa, wang2024loraga, jiang2024mora, lialin2023relora}, while POME applies a Muon-style projection to post-hoc model edits \citep{pome}. PowerMuon differs from both lines of work: it keeps full-parameter finetuning, but changes the optimizer so that updates retain some singular-value information instead of fully flattening it. In this sense, PowerMuon is a training-time spectral bias rather than an explicit low-rank constraint or an after-the-fact projection, matching the effective-rank and curriculum analyses in this paper.

\section{Conclusion}

%Exciting future directions include automatically determining the best $p$ in \ours to use for a given task, and incorporating \ours into a pretraining curriculum 

In this work, we introduced \ours, a Muon-style optimizer that realizes fractional spectral-power updates through an efficient SVD-free bivariate polynomial recurrence. Our results suggest that \ours is effective for finetuning, improving math reasoning and coding, perplexity, and downstream task performance, while also indicating that the optimal spectral geometry differs between pretraining and finetuning. Exciting future directions include automatically selecting the best $p$ for a given task, incorporating \ours into a pretraining curriculum, and developing a deeper understanding of how spectral geometry shapes learning-rate adaptation and downstream performance.

\bibliographystyle{plainnat}
\bibliography{infinite}

\newpage
\appendix 
\newpage
\section{Appendix}

We include omitted proofs, additional experimental details and results.

\subsection{Additional Experimental Details and Results}

%\paragraph{Coefficients for $g(x,y)$ in Eq~\ref{eq:derivatives-special-form}.}

\subsubsection{Further training details.}
\label{sec:training-details}
We tune the optimal hyperparameters such as learning rate and batch size for each optimizer method. We use a cosine learning rate scheduler.
For language modeling, Numina, and OpenCodeInstruct finetuning, the best learning rate is tuned within $[1e-5, 5e-5, 1e-4]$, and the best batch sizes are tuned within $[128, 512, 1024]$. In addition, we tune the number of warmup steps within $[100, 200, 400]$ and a ratio of $0.1$, and tune the minimum learning rate ratio for the cosine scheduler within $[0.1, 0.3]$. For finetuning on GSM8k, due to its smaller dataset size, the best learning rate is tuned within $[2e-5, 5e-5]$, and the batch size is tuned within $[128, 512]$. The exact same seed is used throughout different runs to ensure the same data ordering during training. We use the same number of iterations $N=6$ in Algorithm~\ref{algo:muon-p} and Muon throughout for fair comparison.

Unless otherwise noted, language modeling finetuning on Llama3-1B is trained for 1000 steps, with a global batch size of 32768 tokens. The evaluation is done on a 1000 randomly selected test set. GSM8K is finetuned for 600 steps (10 epochs), given its small dataset size, to avoid catastrophic overfitting; and Numina and OpenCodeInstruct are finetuned for 6000 steps.

We use the ``constitutes", ``courtListener\_opinions", and  ``us\_bills", ``scotus\_filings" components from the Pile of Law dataset, for a total of 625 million tokens. The evaluation is done on the official validation set.

\subsubsection{GSM8K and Numina Evaluation.} To evaluate performance on the validation set of GSM8K and Numina, we evaluate pass@8 and temperature 0.8, and use 32 samples to approximate an unbiased estimator. Evaluation is done on the 1319 problems in the official test split for GSM8K, and the 100 problems in the official test split for Numina.

%Note that we use Muon as our baseline, as that is the optimizer \ours is extending. This also makes a fair comparison, as the pretrained checkpoints are optimized with AdamW, and hence exhibit favorable optimization dynamics towards AdamW. For instance, using AdamW for finetuning means the same type of coordinate-wise rescaling of the gradients is used as in pretraining, which is not the case for Muon or \ours. %This is indeed a limitation of the current state of this work.

\subsubsection{Choice of coefficients of $f(x,y)$.}
\label{sec:c-values}
As discussed in \S\ref{sec:training}, the polynomial $f(x, y)$ is taken to be the simplest form that satisfy the convergence conditions in Lemmas \ref{divisibility-lemma}, \ref{basin-of-attraction}, and ~\ref{lem:choice_of_c}. Once the degree and convergence criteria are satisfied, we find the optimal coefficients for $f(x)$ using a simple grid search within the range of convergence, on a sample of 1000 GPT-2 gradients produced on Fineweb. The grid search aims to minimize the $\ell_2$ distance between the true matrix power $G^p$ and the \ours approximation for the gradient $G$. We found that the range of acceptable convergence on the gradients in practice is larger than in theory due to using a finite number of iterations.

We found grid search to be simpler and more efficient than gradient-descent based search methods for the optimal coefficients.
Table~\ref{tab:f-coefficients} contains the optimal coefficients used for a range of exponents $p$.

\begin{table}[h]
\centering
\small
\begin{tabular}{lcccccc}
\toprule
\textbf{\ours exponent} 
& $\frac{1}{15}$ 
& $\frac{1}{5}$ 
& $\frac{1}{3}$ 
& $\frac{1}{2}$ 
& $\frac{3}{5}$ 
& $\frac{13}{15}$ \\
\midrule
\textbf{$c$} 
& $0.133$ 
& $0.700$ 
& $0.660$ 
& $0.300$ 
& $0.200$ 
& $0.133$ \\
%\midrule
\textbf{$d$} 
& - 
& - 
& - 
& $0.500$ 
& - 
& - \\
\bottomrule
\end{tabular}
\caption{Empirical coefficients $c$ and $d$ in $f(x,y)$ for a range of \ours exponents found with grid search within the feasible range as in Lemma~\ref{lem:choice_of_c}.}
\label{tab:f-coefficients}
\end{table}

The coefficients of the polynomial $f(x, y)$ in Eq~\ref{eq:derivatives-special-form}, such as $c$ in Eq~\ref{eq:f_polynomial}, trade off attracting fixed point guarantee with convergence speed w.r.t. the number of recurrence iterations: smaller values of $c$ makes the fixed point attracting and guarantees convergence, but can decrease convergence speed, especially along the small singular value directions. It is an interesting future direction to vary $c$ for different singular directions.

% Lemma~\ref{divisibility-lemma} and Lemma~\ref{basin-of-attraction} describe convergence conditions when iterating the polynomial $f(x, y)$, including conditions on its coefficients. In practice, we do a simple grid search on sample gradients from Fineweb, and pick the coefficients within the theoretically convergent range that minimize the difference to the true SVD on the gradients. 

In summary, we want the largest coefficients that still guarantee convergence for both large and small singular values. Figure~\ref{fig:choice_of_c} illustrates the convergence v.s. speed tradeoff for a small and a large singular value for $p=\frac{1}{3}$.

\begin{figure}[h]
    \centering    \includegraphics[width=0.9\linewidth]{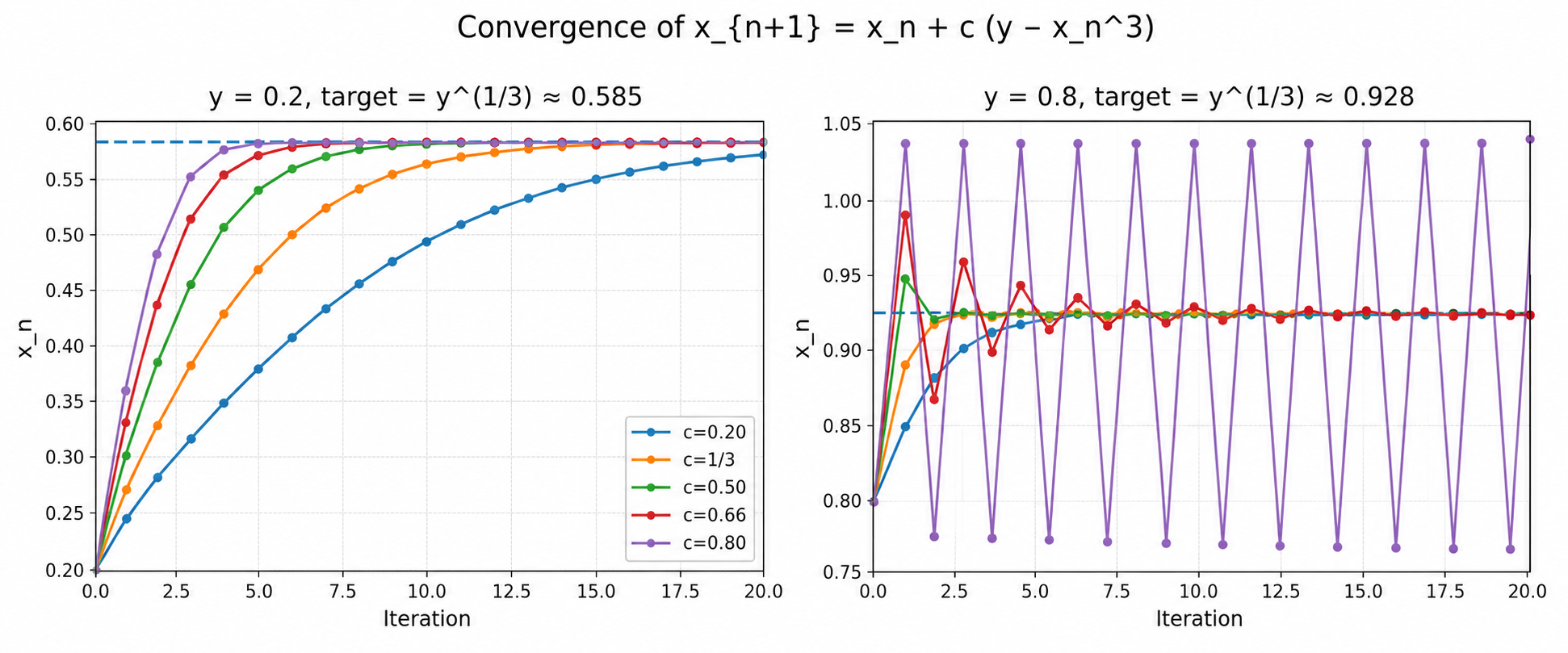}
    \caption{Convergence v.s. speed tradeoff on $c$ for a small and a large singular value for $p=\frac{1}{3}$. While a larger value of $c$ accelerates convergence along small singular directions, it can prevent convergence along larger singular directions.}
    \label{fig:choice_of_c}
    \vspace{-1mm}
\end{figure}

\subsubsection{Comparison with additional baselines.}
\label{sec:additional-baselines}
We compare \ours with stronger baselines that modify Muon's spectrum or preconditioner to improve robustness or convergence: AdaMuon, PolarGrad (both `qdwh' and `zolo-pd' backends), and PolarExpress \citep{si2025adamuon, he2025root, lau2025polargrad, polar_express}. 

We chose the Numina dataset, as it is a much larger scale and more complex math reasoning dataset than GSM8k, and hence offers a more robust signal. We obtained the following results when finetuning Llama3-1B on 393 million tokens of Numina:

\begin{table}[ht]
\centering
\begin{tabular}{lrr}
\toprule
Method & lr=$5e-5$ & lr=$1e-5$ \\
\midrule
Muon & 0.618 & 0.769 \\
\textbf{Muon\textsuperscript{p}} & \textbf{0.597} & \textbf{0.708} \\
PolarGrad zolo-pd & 0.768 & 0.870 \\
PolarGrad qdwh & 0.767 & 0.871 \\
% ROOT & NaN & NaN \\
Polar Express & 0.767 & 0.870 \\
AdaMuon & 0.750 & 0.868 \\
\bottomrule
\end{tabular}
\caption{Evaluation loss by method and learning rate.}
\label{tab:eval-loss}
\end{table}

We tried our best to reproduce the baselines as faithfully as possible, and worked on ensuring baseline stability whenever needed. E.g. PolarGrad with the zolo-pd backend is only numerically stable when executed in float64 due to its Cholesky kernels, which is the precision we used.

Although HTMuon \citep{pang2026htmuon} is most similar to Muon$^p$ in spirit, by applying fractional singular powers $S^p$, we could not compare with HTMuon, as it requires exact SVD, and hence belongs to a different family of algorithms and is orders of magnitude slower, especially at the 1B scale. We also benchmarked ROOT \citep{he2025root}, but encountered numerical instabilities.

It is interesting that Muon outperformed these baselines aimed towards improving its robustness. One hypothesis is that all these baselines were developed and tested in the \textit{pretraining} setting, where having a more accurate polar factor approximation (with singular values 1) developed in these baselines can be beneficial. 

However, as Muon only \textit{approximates} a polar decomposition, indeed Muon gradient singular values are roughly Uniform$(0.5, 1.5)$ instead of exactly 1 (due to its maximizing the polynomial slope at 0), Muon performs more implicit regularization during finetuning compared to these baselines (Indeed Muon's updates are not truly full rank as shown by its effective rank in Figure~\ref{fig:erank}.)

\subsubsection{Performance on larger scale models}
\label{sec:larger-scale}
We aim to study \ours beyond the 1B scale. Specifically, we finetuned Llama3-3B (base) on 786 million tokens of the Numina dataset. 
Here we report test accuracy on the official Numina test split, the pass@8 accuracy is evaluated using temperature$=0.8$ and $32$ generations to approximate an unbiased estimator. As shown in \ref{tab:llama-3b-numina}, \ours consistently outperforms Muon, consistent with the observations on smaller scales.

\begin{table}[h]
\centering
\begin{tabular}{llr}
\toprule
Learning rate & Optimizer & Pass@8 \\
\midrule
\multirow{2}{*}{$1\times 10^{-4}$}
    & Muon$^p$ & \textbf{0.531} \\
    & Muon     & 0.506 \\
\addlinespace
\multirow{2}{*}{$5\times 10^{-5}$}
    & Muon$^p$ & \textbf{0.496} \\
    & Muon     & 0.455 \\
\addlinespace
\multirow{2}{*}{$1\times 10^{-5}$}
    & Muon$^p$ & \textbf{0.361} \\
    & Muon     & 0.296 \\
\bottomrule
\end{tabular}
\caption{Pass@8 performance of Muon and \ours across learning rates after finetuning Llama3-3B with either Muon or \ours on Numina.}
\label{tab:llama-3b-numina}
\end{table}

\subsubsection{Further runtime comparison}
Section~\ref{sec:complexity} compares the complexity between \ours and Muon. 
Here to aim to study the runtimes in practice in terms of wallclock time.
We randomly initialized a Llama3-1B model in bf16, and benchmarked the runtime of a single `optimizer.step' call. We used a single H100 GPU and identical data. We allowed each optimizer to warm up for 3 steps (for `torch.compile', etc), and measured the mean runtime of the subsequent 15 steps.

As shown in Table~\ref{tab:optimizer-step-time-more},
\ours with $p=\frac{1}{5}$ uses a quintic polynomial recurrence, as opposed to the cubic recurrence used when $p=\frac{1}{3}$. AdaMuon's greater latency is due to its extra `sign()` function and Adam-style second-moment EMA.

\begin{table}[h]
\centering
\begin{tabular}{l S[table-format=3.2]S[table-format=1.2]}
\toprule
\textbf{Backend} & \multicolumn{2}{c}{\textbf{Optimizer step time (ms)}} \\
\cmidrule(lr){2-3}
 & {\textbf{Mean}} & {\textbf{Std.}} \\
\midrule
$\mathrm{Muon}^{p}$, $p=\frac{1}{3}$ & 101.89 & 1.22 \\
$\mathrm{Muon}^{p}$, $p=\frac{1}{5}$ & 116.30 & 0.43 \\
$\mathrm{Muon}$                       & 120.73 & 0.26 \\
Polar Express                          & 121.44 & 0.05 \\
AdaMuon                                & 150.53 & 1.26 \\
\bottomrule
\end{tabular}
\caption{Optimizer step time by backend. Values are reported as mean and standard deviation over 15 steps.}
\label{tab:optimizer-step-time-more}
\end{table}

\subsubsection{Further curriculum results}
Figure~\ref{fig:curriculum-more} shows additional Muon $\to$ \ours curriculum results for pretraining on Fineweb.

\begin{figure}[!h]
    \centering    \includegraphics[width=0.4\linewidth]{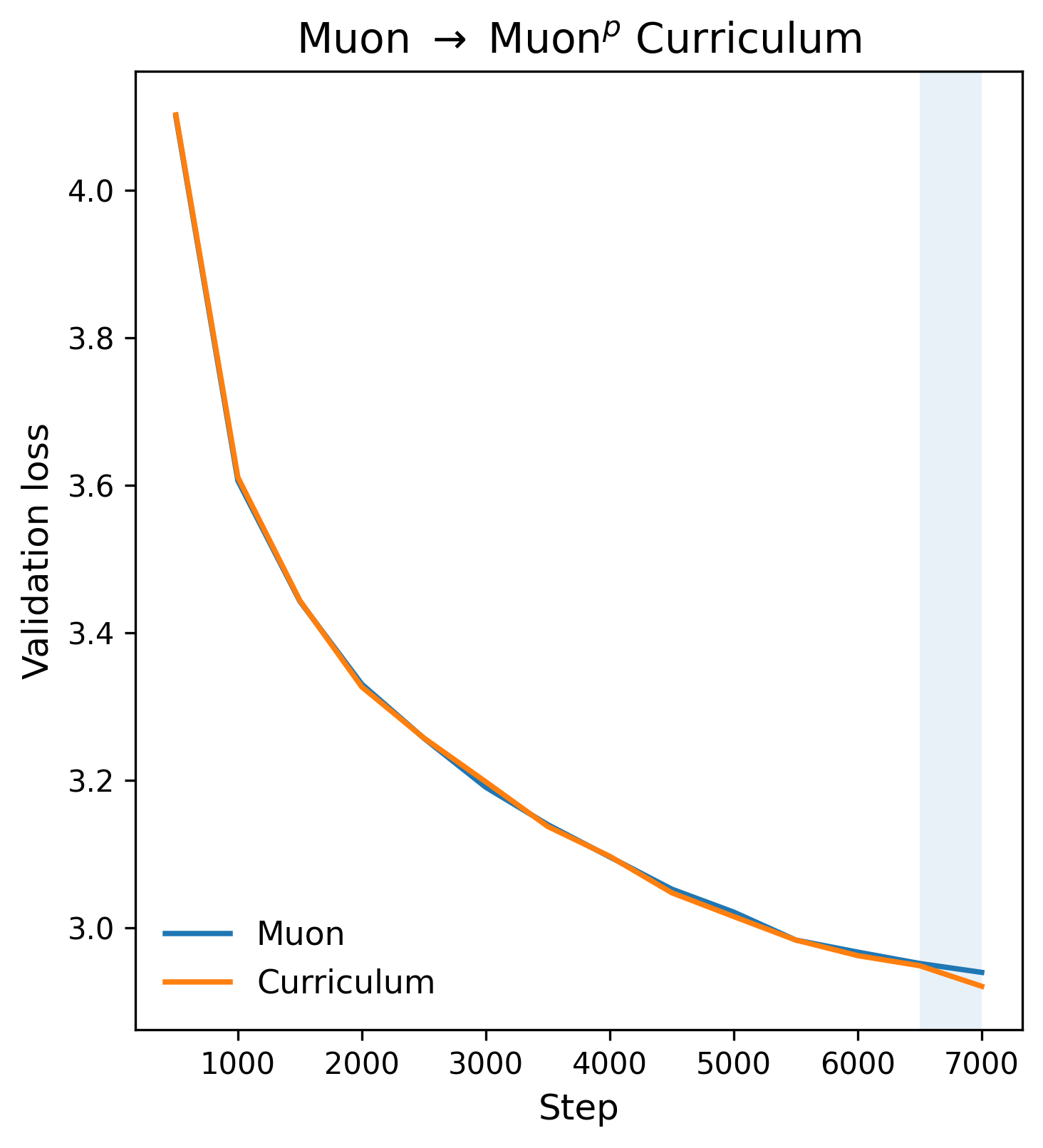}
    \caption{Muon $\to$ \ours curriculum when pretrained on Fineweb.}
    \label{fig:curriculum-more}
    \vspace{-1mm}
\end{figure}

\newpage

\subsection{Proofs}

We begin with the proofs of Lemma \ref{no-go} and Theorem \ref{odd-matrix-polynomial}. We then prove the results of Section \ref{sec-derivatives}, before proving Theorem \ref{main-existence}, which relies on these results. Afterwards, we prove Lemma \ref{steepest-descent}.

\begin{lemma}[Lemma \ref{no-go}, Iterating single-variable polynomials does not compute rational powers]For any real number $p\in (0,1)$, there does not exist any polynomial in one variable $f$ such that, for all invertible matrices $G = U S V^\top$ with all singular values at most $1$, if we let $S_0 =S$ and $S_{n+1} = f(S_n)$ for all $n \geq 0$, then we have $\lim_{n\to\infty}  U f(S_n)  V^\top = U  S^{a/b} V^\top $. \end{lemma}

\begin{proof}[Proof of Lemma \ref{no-go}] Fix $p$, and suppose for contradiction there did exist such an $f$. Since $\lim_{n\to\infty}  U f(S_n)  V^\top = U  S^{p} V^\top $, multiplying by $U^\top$ on the left and $V$ on the right,  we have $\lim_{n\to\infty}   f(S_n)   =  S^{p} $. Let $x$ be the top-left diagonal entry of $S$ and let $x_n$ be the top-left diagonal entry of $S_n$. Then we have $\lim_{n\to\infty}  x_n = x^{p}$, we have $x_0=x$, and we have $x_{n+1} = f(x_n)$ for all $n$. Since $f$ is continuous we have
$$ x^{a/b} \lim_{n\to\infty} x_n= \lim_{n\to\infty} x_{n+1} = \lim_{n\to\infty} f(x_n) = f( \lim_{n\to\infty} x_n) = x^{p}.$$

Thus the polynomial function $f(y)-y$ vanishes at $y=x^{p}$. This holds for invertible matrix $U S V^\top$ with all singular values at most $1$ and hence holds for all $x$ in $(0,1]$. Since $x^{p}$ for $x$ in $(0,1]$  includes infinitely many distinct values of $y$, the polynomial function $f(y)-y$ is identically $0$. Thus $f(x)=x$ so $S_{n+1}=S_n$ for all $n$ and thus $S_n=S$ for all $n$ by induction, hence $$\lim_{n\to\infty} U  S_n V^\top = \lim_{n\to\infty} U S V^\top = U S V^\top \neq U S^{p} V,$$ contradicting our assumption.
\end{proof}

%Under such an iteration, the singular values of the matrix can only converge to the fixed points of the polynomial $f$. A polynomial can only have finitely many fixed points, but the diagonal entries of $S^{a/b}$ can take arbitrary real values in $(0,1]$.

\begin{theorem}[Theorem \ref{odd-matrix-polynomial}, Computing odd two-variable polynomials] 

\begin{enumerate} \item Let $f(x,y)$ be a polynomial in two variables with real coefficients which is odd. We can write $f$ in the form \begin{equation}\label{app-odd-polynomial-expression} f(x,y) = \sum_{i, j =0}^d a_{2i+1, 2j} x^{2i+1} y^{2j} + \sum_{i,j=0}^d a_{2i, 2j+1} x^{2i} y^{2j+1}\end{equation} for some natural number $d$ and tuples of real numbers $(a_{2i+1,j})_{i,j=0}^d , (a_{2i,2j+1} )_{i,j=0}^d$.
\item For $f$ a polynomial given by the formula \eqref{app-odd-polynomial-expression}, for $G = U S V^\top$ and $G_n = U S_n V^\top$ we have 
    \[ U f(S_n, S) V^\top =  \sum_{i, j } a_{2i+1, 2j} G_n (G_n^\top G_n)^i ( G^\top G)^j  + \sum_{i,j} a_{2i, 2j+1}  G (G_n^\top G_n)^i (G^\top G)^j .\]
    \end{enumerate}
\end{theorem}

\begin{proof} For part (1), since $f(x,y)$ is a polynomial, we can write \begin{equation}\label{app-arbitrary-polynomial} f(x,y) = \sum_{i,j=0}^{2d+1} a_{i,j} x^i y^j.\end{equation} Then $$f(-x,-y)= \sum_{i,j=0}^{2d+1} (-1)^i (-1)^j a_{i,j} x^i y^j.$$ For $f(-x,-y)$ to equal $-f(x,y)$ for all $x,y$, the coefficient  $(-1)^i (-1)^j a_{i,j}$ of $x^i y^j$ in $f(-x,-y)$ must equal te coefficient $- a_{i,j}$ of $x^i y^j$ in $-f(x,y)$. If $i+j$ is even this forces $a_{i,j}=-a_{i,j}$ and thus $a_{i,j}=0$. Hence $a_{i,j}$ is nonzero only if $i+j$ is odd, which forces one of $i,j$ to be even and the other to be odd. Restricting the sum \eqref{app-arbitrary-polynomial} to only the terms with $i$ odd and $j$ even or $i$ even and $j$ odd, we obtain \eqref{app-odd-polynomial-expression}. 

For part (2), we plug in $G = U S V^\top$  and $G_n = U S_n V^\top$ and use $U^\top U =I$ and then $V^\top V= I$ to obtain 
\[  \sum_{i, j } a_{2i+1, 2j} G_n (G_n^\top G_n)^i ( G^\top G)^j  + \sum_{i,j} a_{2i, 2j+1}  G (G_n^\top G_n)^i (G^\top G)^j \]
\[ = \sum_{i, j } a_{2i+1, 2j} U S_n V^\top  (V S_n U^\top U S_n V^\top )^i ( V S U^\top  U S V^\top )^j  \] \[+ \sum_{i,j} a_{2i, 2j+1}  U S V^\top (V S_n U^\top U S_n V^\top)^i (V S U^\top U S V^\top )^j \]
\[=\sum_{i, j } a_{2i+1, 2j} U S_n V^\top  (V S_n^2 V^\top )^i ( V S^2 V^\top )^j  + \sum_{i,j} a_{2i, 2j+1}  U S V^\top (V S_n^2 V^\top)^i (V S^2 V^\top )^j \]
\[=\sum_{i, j } a_{2i+1, 2j} U S_n S_n^{2i}  S^{2j} V^\top  + \sum_{i,j} a_{2i, 2j+1}  U S  S_n^{2i}  S^{2j} V^\top  \] \[= \sum_{i, j } a_{2i+1, 2j} U S_n^{2i+1}  S^{2j} V^\top  + \sum_{i,j} a_{2i, 2j+1}  U   S_n^{2i}  S^{2j+1} V^\top  \]
\[ = U \left(  \sum_{i, j } a_{2i+1, 2j} S_n^{2i+1}  S^{2j}   + \sum_{i,j} a_{2i, 2j+1}     S_n^{2i}  S^{2j+1} \right) V^\top = U f(S_n, S) V^\top.\]

\end{proof}

\begin{lemma}[Lemma \ref{divisibility-lemma}, fixed point criterion] Let $a$ and $b$ be two coprime positive integers. Let $f$ be a real polynomial in two variables. The following are equivalent:
\begin{enumerate}
    \item We have $f(y^{a/b}, y)= y^{a/b}$ for all $y\in (0,1]$.
    \item We have $f(y^{a/b}, y)=y^{a/b}$ for all $y \geq 0$.
    \item There exists a polynomial $g(x,y)$ such that we have $f(x,y) =x + (y^a-x^b) g(x,y)$.
\end{enumerate}
\end{lemma}

\begin{proof} (2) implies (1) by restriction. (3) implies (2) by substitution: 

$$ f(y^{a/b}, y) = y^{a/b} + (y^a - (y^{a/b})^b) g(y^{a/b},y)= y^{a/b} + (y^a-y^a) g(y^{a/b},y)= y^{a/b}+0=y^{a/b}.$$

It remains to check that (1) implies (3). By polynomial long division, we can write $f(x,y) = x + (y^a-x^b) g(x,y) + r(x,y)$ for a polynomial $g(x,y)$ and a remainder $r(x,y)$ of degree $\leq a-1$ in $y$, i.e. every monomial appearing in $r$ has degree at most $a-1$ in $y$.  Since $f (y^{a/b}, y)=y^{a/b} $ for all $y \in (0,1]$, we have $f(z^a, z^b) = z^a$ for all $z \in (0,1]$ by taking $y=z^b$. Substituting in we get

$$ z^a = f(z^a,z^b) = z^a + ((z^a)^b - (z^b)^a) g(z^a,z^b) + r(z^a,z^b) =  z^a + 0 + r(z^a,z^b)$$
so $r(z^a, z^b)=0$ for all $z \in (0,1]$. Since $r(z^a,z^b)$ is a polynomial in $z$, this makes $r(z^a,z^b)$ identically zero. If we write
$$ r(x,y) = \sum_{i=0}^{d} \sum_{j=a-1}^d c_{i,j} x^i y^j $$
we get
$$ r(z^a,z^b) =  \sum_{i=0}^{d} \sum_{j=0}^{a-1} c_{i,j} z^{ ia +jb }.$$

However, every integer $n$ has a unique expression as $ia+jb$ for an integer $i$ and $j$ an integer between $0$ and $a-1$, since $j$ must be congruent to $n$ times the modular inverse of $b$ modulo $a$ and there is a unique integer between $0$ and $a-1$ in that congruence class, and $i$ is determined by $j$ and the equation  $ia+jb=n$. So each coefficient of $r(z^a,z^b)$ is the sum of at most one $c_{i,j}$, and hence that $c_{i,j}$ must vanish, thus every $c_{i,j}$ vanishes so $r=0$, as desired. \end{proof}

To further ensure that $f$ is odd, we must adapt the formula of Lemma \ref{divisibility-lemma} depending on whether $a$ and $b$ are odd. Specifically, we take
\begin{equation}\label{odd-formula} f(x,y) =x  + (y^a - x^b) g(x,y) \textrm{ with }g \textrm{ even}\end{equation} 
if $a$ and $b$ are both odd and
\begin{equation}\label{even-formula}  f(x,y) = x+ (y^{2a}- x^{2b}) h(x,y) \textrm{ with } h \textrm{ odd } \end{equation} 
if $a$ or $b$ is even. Note that \eqref{even-formula} is the special case of \eqref{odd-formula} where $g(x,y) = (y^a+x^b) h(x,y)$.

\begin{lemma}[Lemma \ref{odd-even-lemma}, oddness criterion] Let $a$ and $b$ be two positive integers. \begin{enumerate}
    \item If $a$ and $b$ are odd, the expression \eqref{odd-formula} is always odd.
    \item The expression \eqref{even-formula} is always odd.
\end{enumerate}
\end{lemma}
\begin{proof} These both follow from the facts that an odd polynomial times an even polynomial is odd and the sum of two odd polynomials is odd. For (1), we observe that $y^a-x^b$ is odd and $g$ is even by assumption so their product is odd and then adding $x$ preserves oddness. For (2), we observe that $y^{2a}-x^{2b}$ is even and $h$ is odd by assumption so their product is odd and then adding $x$ preserves oddness. \end{proof}

\begin{lemma}[Lemma \ref{basin-of-attraction}, attracting fixed point criterion] Let $f$ be a polynomial in two variables and let $a$ and $b$ be two positive integers.   Say $y^{a/b}$ is an \emph{attracting fixed point} of $f(x,y)$ if $f(y^{a/b},y) = y^{a/b}$ and $$ \left|  \frac{\partial f}{\partial x} ( y^{a/b}, y) \right| <1. $$ 

If $y^{a/b}$ is an attracting fixed point of $f(x,y) $ then there exists an interval $I \subset \mathbb R$ containing $y^{a/b}$ such that, for any $x_0 \in I$, if we set $x_{n+1} = f(x_n,y)$ for all $n \geq 0$, then $\lim_{n\to\infty} x_n = y^{a/b}$.
\end{lemma}

\begin{proof} This is a standard result in dynamical systems. To prove it, take $\delta$ such that $\left|  \frac{\partial f}{\partial x} ( y^{a/b}, y) \right| \leq  1-2\delta .$ By Taylor's theorem, we have

$$ f(x,y)= f(y^{a/b}, y) + (x-y^{a/b} )  \frac{\partial f}{\partial x} ( y^{a/b}, y)  + O ( | x- y^{a/b}|^2) $$ $$= y^{a/b} + (x-y^{a/b} )  \frac{\partial f}{\partial x} ( y^{a/b}, y)  + O ( | x- y^{a/b}|^2)$$
so

$$\left| f(x,y) - y^{a/b} \right| \leq \left| x-y^{a/b} \right|   \left| \frac{\partial f}{\partial x} ( y^{a/b}, y) \right|   + O ( | x- y^{a/b}|^2)$$ $$\leq \left| x-y^{a/b} \right|   (1-2 \delta)    + O ( | x- y^{a/b}|^2) $$

We can choose $C$ such that if $| x- y^{a/b}| \leq C$ the $O ( | x- y^{a/b}|^2)$ term is at most $\delta |x-y^{a/b}|$ which gives $\left| f(x,y) - y^{a/b} \right| \leq (1-\delta) |x-y^{a/b}| .$ Now if $|x_0 - y^{a/b} |\leq C$ we have $|x_n - y^{a/b} | \leq (1-\delta)^n C$ by induction on $n$ so $\lim_{n\to\infty} x_n = y^{a/b}$, and thus we can take $I = (Y^{a/b} - C, y^{a/b}+C)$.\end{proof}

\begin{lemma}[Lemma \ref{derivative-criterion}, choice of $c$] Let $a$ and $b$ be two positive integers. \begin{enumerate}
    \item Let $f(x,y)$ be a polynomial given by the formula \eqref{odd-formula}.  If 
    $$ 0< g( y^{a/b},y) < \frac{2}{ b y^{ \frac{ a(b-1)}{b}}} $$ for all $y\in (0,1]$ then $y^{a/b}$ is an attracting fixed point of $f(x,y)$ for all $y\in (0,1]$.
 \item Let $f(x,y)$ be a polynomial given by the formula \eqref{even-formula}.  If 
    $$ 0< g( y^{a/b},y) < \frac{1}{ b y^{ \frac{ a(2b-1)}{b}}} $$ for all $y\in (0,1]$ then $y^{a/b}$ is an attracting fixed point of $f(x,y)$ for all $y\in (0,1]$.
    \end{enumerate}
\end{lemma}

\begin{proof} This follows from computing the derivative. For $f$ given by \eqref{odd-formula} we have
$$\frac{\partial f}{\partial x} = 1+ (y^a - x^b) \frac{\partial g}{\partial x} (x,y) - b x^{b-1} g(x,y) $$ and substituting $x = y^{a/b}$ we get

$$\frac{\partial f}{\partial x} ( y^{a/b},y) = 1 + (y^a - y^a) \frac{\partial g}{\partial x} (y^{a/b},y) - b y^{\frac{ a (b-1)}{b} } g(y^{a/b} ,y) = 1 - b y^{\frac{ a (b-1)}{b} } g(y^{a/b} ,y)$$
so to have $\left| \frac{\partial f}{\partial x} ( y^{a/b},y)\right|<1 $ it suffices to have 
$$ 0 < b y^{\frac{ a (b-1)}{b} } g(y^{a/b} ,y) < 2$$
and dividing by $b y^{\frac{ a (b-1)}{b} } $ gives part (1).

For $f$ given by \eqref{even-formula} we have
$$\frac{\partial f}{\partial x} = 1+ (y^{2a} - x^{2b}) \frac{\partial h}{\partial x} (x,y) -2 b x^{b-1} h(x,y) $$ and substituting $x = y^{a/b}$ we get

$$\frac{\partial f}{\partial x} ( y^{a/b},y) = 1 + (y^{2a} - y^{2a}) \frac{\partial h}{\partial x} (y^{a/b},y) - 2b y^{\frac{ a (2b-1)}{b} } h(y^{a/b} ,y) = 1 - 2b y^{\frac{ a (2b-1)}{b} } h(y^{a/b} ,y)$$
so to have $\left| \frac{\partial f}{\partial x} ( y^{a/b},y)\right|<1 $ it suffices to have 
$$ 0 < 2b y^{\frac{ a (2b-1)}{b} } h(y^{a/b} ,y) < 2$$
and dividing by $2b y^{\frac{ a (2b-1)}{b} } $ gives part (2).
\end{proof}

Lemma \ref{derivative-criterion} makes it easy to construct polynomials which satisfy the hypothesis of Lemma \ref{basin-of-attraction} for all $y\in (0,1]$. For $a,b$ odd, we can simply take $g(x,y) = c$ for any positive $c< 2/b$.  For $a$ or $b$ even, we can take $h(x,y) = c_0x+c_1y $ for any positive $c_0,c_1$ with $c_0+c_y < 1/b$. 

We have already seen some examples in the odd case. The simplest example of our construction in the even case is:
$$ f(x,y) = x + (y^2 - x^4) ( c_0x + c_1 y) \textrm{ converging to } y^{1/2} $$

We are now ready to prove Theorem \ref{main-existence}.

\begin{proof}[Proof of Theorem \ref{main-existence}] We may assume $a$ and $b$ are coprime without loss of generality by dividing any common factors from $a$ and $b$.

We first check that it suffices to have $f(y^{a/b},y)= y^{a/b}$ and $ x < f(x,y) <  y^{a/b}$ for any $y \in (0,1]$ and $y \leq  x < y^{a/b}$. Indeed, by induction on $n$ we will have $ y \leq x_n < y^{a/b}$ for all $n$, since this is satisfied for $n=0$ and assuming it is satisfied at $n$ we have $x_{n+1} = f(x_n,y)$ and $$ y \leq x_n < f(x_n, y) < y^{a/b}$$  so $$y \leq x_{n+1} < y^{a/b},$$ completing the induction step. Then by assumption we have $x_{n+1} >x_n$ for all $n$, so $x_n$ is a bounded monotone sequence and thus convergences.  Since $f$ is continuous if we set $x = \lim_{n\to\infty} x_n$ we must have
\[ x= \lim_{n\to\infty} x_n =\lim_{n\to\infty} x_{n+1} =\lim_{n\to\infty} f(x_n,y)= f(\lim_{n\to\infty} x_n,y)= f(x,y) \] but we have $y \leq x \leq y^{a/b}$ as we have $y \leq x_n \leq y^{a/b}$ for each $n$. If $x\neq  y^{a/b}$ then we have  $y \leq x < y^{a/b}$ so $x< f(x,y)$ which contradicts $x=f(x,y)$, so we must have $x = y^{a/b}$, as desired.

We now ensure that the sufficient conditions above are satisfied. For $a,b$ odd we take $f(x,y) = x + c (y^a-x^b)$ for some $c>0$. Since $c$ is even, this falls into \eqref{odd-formula}. We saw in Lemma \ref{odd-even-lemma} that $f$ is odd and in Lemma \ref{divisibility-lemma} that $f(y^{a/b},y)=y^{a/b}$. If $y \leq x < y^{a/b}$ then $y^a -x^b >0$ so $f(x,y)>x$. Thus it only remains to ensure that $f(x,y) < y^{a/b}$. We have
\[ y^{a/b} -f(x,y) \] \[= y^{a/b}-x - c (y^a - x^b)= (y^{a/b}-x) (1 - c ( y^{a-a/b} + y^{a-2a/b} x + y^{a-3a/b} x^2 + \dots + x^{b-1})).\]
Since we assumed $x \leq y^{a/b} \leq 1$, the terms $y^{a-a/b} , y^{a-2a/b} x , y^{a-3a/b} x^2 , \dots , x^{b-1}$ are each at most $1$ and there are $b$ terms so \[ y^{a-a/b} + y^{a-2a/b} x + y^{a-3a/b} x^2 + \dots + x^{b-1} \leq b \] and so to have $y^{a/b} -f(x,y) >0$ it suffices to have $c<1/b$. Thus any $c \in (0,1/b)$ works.

For $a$ or $b$ even we take $f(x,y) = x + cy ( y^{2a}-x^{2b})$ for some $c>0$. Since $c$ is odd, this falls into \eqref{even-formula}. We saw in Lemma \ref{odd-even-lemma} that  that $f$ is odd and in Lemma \ref{divisibility-lemma} that $f(y^{a/b},y)=y^{a/b}$. If $y \leq x < y^{a/b}$ then $y^a -x^b >0$ and we assumed $y>0$ so $f(x,y)>x$. Thus it only remains to ensure that $f(x,y) < y^{a/b}$. We have
\[ y^{a/b} -f(x,y) = y^{a/b} -x - c y(y^{2a} -x^{2b})$$ $$ = (y^{a/b}-x) (1-c y ( y^{2a-a/b} + y^{2a-2a/b} x + y^{2a-3a/b} x^2 + \dots + x^{2b-1})).\]
Since we assumed $x \leq y^{a/b} \leq 1$, the terms $ y^{2a-a/b} , y^{2a-2a/b} x , y^{2a-3a/b} x^2 , \dots , x^{2b-1})$ are each at most $1$ and there are $2b$ terms so 
\[ ( y^{2a-a/b} + y^{2a-2a/b} x + y^{2a-3a/b} x^2 + \dots + x^{2b-1}) \leq 2b\] and so to have $y^{a/b} -f(x,y) >0$ it suffices to have $c< 1/(2b)$. Thus any $c\in (0,1/(2b))$ works. \end{proof}

\begin{theorem}[Theorem \ref{steepest-descent}] For any real number $p \in (0,1)$ and let $q=1+1/p$. Let $G$ be an $n\times n$ matrix with singular value decomposition $US V^\top$. Let $\langle, \rangle$ denote the entrywise dot product of matrices and let  $||\cdot||_q$ denote the Schatten $q$-norm. Then for every $c_1>0$, the maximum value of $\langle G,M\rangle$ among $n\times n$ matrices $M$ with $||M||_q\leq c_1$ is attained by $M = c_2 U S^{p} V^\top $ for some $c_2>0$. \end{theorem}
\begin{proof} \label{steepest-descent-proof}For any $n\times n$ matrix $D$ with $||M||_q\leq c_1$, let $s_1,\dots,s_n$ be the singular values of $G$. Then we have
$$\langle G, M\rangle = \operatorname{tr}(G M^\top) \leq ||G||_{p+1}  ||M||_q \leq ||G||_{p+1} c_1  = (\sum_{i=1}^n s_i^{p+1} )^{\frac{1}{p+1}} c_1 $$
by H\"older's inequality for the Schatten $p$-norms, since $\frac{1}{p+1} +\frac{1}{q}=1$.

%\sum_{i=1}^n \sigma_i(G D_\top) \leq \sum_{i=1}^n \sigma_i(G) \sigma_i(D^\top) = \sum_{i=1}^n \sigma_i(G) \sigma_i(D) = \sum_{i=1}^n s_i d_i $$
%by \cite[Theorem 3.13(a')]{Horn1991} and \cite[Theorem 3.3.14(b)]{Horn1991}. By Holder's inequality if $||D||_q \leq c_1$ we have
%$$\langle G, D\rangle  \leq ( \sum_{i=1}^n s_i^{p+1} )^{\frac{1}{p+1}} (\sum_{i=1}^n d_i^{ \frac{p+1}{p}} )^{\frac{p}{p+1}} =  ( \sum_{i=1}^n s_i^{p+1} )^{\frac{1}{p+1}} (\sum_{i=1}^n d_i^{ q} )^{\frac{1}{q}}$$ $$ =( \sum_{i=1}^n s_i^{p+1} )^{\frac{1}{p+1}} ||D||_q \leq   ( \sum_{i=1}^n s_i^{p+1} )^{\frac{1}{p+1}} (\sum_{i=1}^n d_i^{ q} )^{\frac{1}{q}} =( \sum_{i=1}^n s_i^{p+1} )^{\frac{1}{p+1}} c_1 .$$
On the other hand, for $c_2 = \frac{c_1}{ (\sum_{i=1}^n s_i^{pq})^{\frac{1}{q}}}$ and $M = c_2  U S^{p} V^\top$, we have $$||M||_q = c_2   (\sum_{i=1}^n s_i^{pq})^{\frac{1}{q}} = \frac{c_1}{ (\sum_{i=1}^n s_i^{pq})^{\frac{1}{q}}}(\sum_{i=1}^n s_i^{pq})^{\frac{1}{q}}  =c_1$$ and we have 
$$\langle G,M\rangle = c_2 \operatorname{tr} ( U S V^\top V S^p U^\top) = c_2 \operatorname{tr} (S ^{1+p}) = c_2 \sum_{i=1}^n s_i^{1+p}$$ $$ = c_1 \frac{ \sum_{i=1}^n s_i^{1+p}}{( \sum_{i=1}^n s_i^{pq})^{\frac{1}{q}}} = c_1 \frac{ \sum_{i=1}^n s_i^{1+p}}{( \sum_{i=1}^n s_i^{1+p})^{\frac{p}{p+1}}}=c_1 ( \sum_{i=1}^n s_i^{1+p})^{\frac{1}{p+1} }$$ so the maximum is attained by this value of $D$, as desired. 
\end{proof}

% gsm8k evals: --k 8 --n_samples 32 --temperature 0.8

%\begin{figure}[h]
%    \centering
%    \includegraphics[width=.7\linewidth]{plots/finetune_eval_diff_p.png}
%    \caption{Diff p.}
%    \label{fig:diff_p}
%    \vspace{-1mm}
%\end{figure}

%courtListener_opinions us_bills scotus_filings 

\end{document}